\documentclass{article}

\usepackage{microtype}
\usepackage{graphicx}
\usepackage{subfigure}
\usepackage{booktabs} % for professional tables

\usepackage{hyperref}

\usepackage[accepted]{icml2024}

\usepackage{amsmath}
\usepackage{amssymb}
\usepackage{mathtools}
\usepackage{amsthm}

\usepackage[capitalize,noabbrev]{cleveref}

\theoremstyle{plain}
\newtheorem{theorem}{Theorem}[section]

\newtheorem{lemma}[theorem]{Lemma}
\newtheorem{corollary}[theorem]{Corollary}
\theoremstyle{definition}
\newtheorem{definition}[theorem]{Definition}

\theoremstyle{remark}
\newtheorem{remark}[theorem]{Remark}

\usepackage[textsize=tiny]{todonotes}

\icmltitlerunning{Towards Theoretical Understandings of Self-Consuming Generative Models}

\begin{document}

\twocolumn[
\icmltitle{Towards Theoretical Understandings of Self-Consuming Generative Models}

% It is OKAY to include author information, even for blind
% submissions: the style file will automatically remove it for you
% unless you've provided the [accepted] option to the icml2024
% package.

% List of affiliations: The first argument should be a (short)
% identifier you will use later to specify author affiliations
% Academic affiliations should list Department, University, City, Region, Country
% Industry affiliations should list Company, City, Region, Country

% You can specify symbols, otherwise they are numbered in order.
% Ideally, you should not use this facility. Affiliations will be numbered
% in order of appearance and this is the preferred way.

\begin{icmlauthorlist}
\icmlauthor{Shi Fu}{ustc}
\icmlauthor{Sen Zhang}{usyd}
\icmlauthor{Yingjie Wang}{sch,erc}
\icmlauthor{Xinmei Tian}{ustc,hefei}
\icmlauthor{Dacheng Tao}{ntu}
\end{icmlauthorlist}

\icmlaffiliation{ustc}{University of Science and Technology of China, Hefei, China}
\icmlaffiliation{usyd}{The University of Sydney, Sydney, Australia}
\icmlaffiliation{sch}{College of Control Science and Engineering, China University of Petroleum (East China), Qingdao, China}
\icmlaffiliation{erc}{Engineering Research Center of Intelligent Technology for Agriculture, Ministry of Education, China}
\icmlaffiliation{ntu}{Nanyang Technological University, Singapore}
\icmlaffiliation{hefei}{Institute of Artificial Intelligence, Hefei Comprehensive National Science Center, China}

\icmlcorrespondingauthor{Shi Fu}{fs311@mail.ustc.edu.cn}
\icmlcorrespondingauthor{Dacheng Tao}{dacheng.tao@gmail.com}

% You may provide any keywords that you
% find helpful for describing your paper; these are used to populate
% the "keywords" metadata in the PDF but will not be shown in the document
\icmlkeywords{Machine Learning, ICML}

\vskip 0.3in
]

% this must go after the closing bracket ] following \twocolumn[ ...

% This command actually creates the footnote in the first column
% listing the affiliations and the copyright notice.
% The command takes one argument, which is text to display at the start of the footnote.
% The \icmlEqualContribution command is standard text for equal contribution.
% Remove it (just {}) if you do not need this facility.

\printAffiliationsAndNotice{}  % leave blank if no need to mention equal contribution
%\printAffiliationsAndNotice{\icmlEqualContribution} % otherwise use the standard text.

\begin{abstract}
This paper tackles the emerging challenge of training generative models within a self-consuming loop, wherein successive generations of models are recursively trained on mixtures of real and synthetic data from previous generations. We construct a theoretical framework to rigorously evaluate how this training procedure impacts the data distributions learned by future models, including parametric and non-parametric models. Specifically, we derive bounds on the total variation (TV) distance between the synthetic data distributions produced by future models and the original real data distribution under various mixed training scenarios for diffusion models with a one-hidden-layer neural network score function. Our analysis demonstrates that this distance can be effectively controlled under the condition that mixed training dataset sizes or proportions of real data are large enough. Interestingly, we further unveil a phase transition induced by expanding synthetic data amounts, proving theoretically that while the TV distance exhibits an initial ascent, it declines beyond a threshold point. Finally, we present results for kernel density estimation, delivering nuanced insights such as the impact of mixed data training on error propagation.
\end{abstract}

\section{Introduction}
With the rapid advancements in deep generative models, synthetic data of all varieties is expanding swiftly. Notably, publicly accessible generative models, such as Stable Diffusion \cite{rombach2022high} for images and ChatGPT \cite{OpenAI_2023} for text, have directly enabled the creation and dissemination of synthetic content at scale, thereby accelerating the flow of synthetic data towards the Internet. Consequently, this surge in synthetic data has led to a situation where even existing web-scale datasets are known to contain generated content \cite{schuhmann2022laion}. Additionally, the identification of such generated content introduces distinctive technical challenges \cite{sadasivan2023can,huschens2023you}. 

Despite potential risks, synthetic data is also being deliberately leveraged in various applications for several reasons. Firstly, generating synthetic training data offers a more efficient alternative to sourcing real-world samples, and has been shown to improve model performance through data augmentation \cite{antoniou2017data,azizi2023synthetic}. Secondly, in sensitive domains such as medical imaging, synthetic data enables critical privacy protection \cite{dumont2021overcoming}. 
More importantly, the expanding scale of deep generative models necessitates synthetic data, as these models are now trained on web-scale datasets that likely exhaust the supply of readily available real data on the internet \cite{villalobos2022will}. Thus, future generations of deep generative models will inevitably need to confront the presence of synthetic data in their training datasets. Consequently, a self-consuming training loop emerges in which future models are repeatedly trained on synthetic data generated from previous generations.

The study of generative models within the self-consuming loop has attracted substantial attention in current research. Empirical results from \citet{shumailov2023curse} and \citet{briesch2023large} suggest that output diversity inevitably decreases after sufficient training generations. \citet{alemohammad2023self} perform experiments under various mixed training scenarios and conclude that injecting real data can mitigate model collapse. Despite these empirical observations, however, theoretical insights are still lacking. \citet{shumailov2023curse} and \citet{alemohammad2023self} provide theoretical intuition by analyzing simple Gaussian toy models, but their approach targets intuitive understanding rather than in-depth analysis. \citet{bertrand2023stability} further establishes an upper bound on the deviation of the output parameters of the likelihood-based generative model from optimal values. However, a key limitation of \citet{bertrand2023stability} is their direct assumption on the upper bounds of optimization errors and statistical errors resulting from finite sampling, rather than providing a rigorous theoretical analysis. Furthermore, their theoretical results are limited to parameter differences when training on mixed datasets comprising real data and synthetic data generated solely from the most recent generative model.  

In contrast, our work aims to provide a comprehensive theoretical understanding of how training generative models, such as diffusion models and kernel density estimators, within self-consuming loops on various mixed datasets affects the fidelity of learned data distributions. We analyze this issue by exploring more general and diverse compositions of training data, going beyond simplistic assumptions. Additionally, we tackle the challenge of distributional discrepancy between synthetic and real-world data, transcending the limitations of analyzing model parameter discrepancy. Moreover, to overcome the direct assumptions on the bounds of statistical and optimization errors, we conduct tailored analyses of the training dynamics for specific generative models, including kernel density estimators and simplified diffusion models. Ultimately, this enables us to establish upper bounds on the TV distance. By eschewing simplistic assumptions and undertaking more nuanced analyses, our work offers key insights into the dynamics within this rapidly evolving domain of self-consuming generative modeling. The main contributions of this work include:

1. We propose a theoretical framework to assess the impact of training generative models within self-consuming loops. Specifically, We derive TV distance bounds between the original and the future learned data distributions under various mixed training scenarios for diffusion models with a one-hidden-layer neural network score function.

2. We provide requirements on sample sizes and proportions of real data to control errors. Notably, for the most extreme case of full synthetic data, we demonstrate the necessity of quartic sample growth or incorporating $\Omega((i-1)/i)$ proportions of real data in the final generation to restrict errors, where $i$ denotes the number of training generations.

3. We analyze the dynamics of increasing synthetic data on error propagation, unveiling a phase transition as synthetic data expands while real data remains fixed. Interestingly, before this transition point, more synthetic data impairs performance. However, beyond the threshold, incorporating additional synthetic data enhances performance.

%4. We apply the theoretical results to diffusion models by analyzing their training dynamics. We further establish upper bounds that avoid the curse of dimensionality when early-stopped.

4. We present a theoretical analysis of self-consuming non-parametric generative models, particularly deriving TV distance bounds between future learned synthetic data distributions and the original real data distribution using kernel density estimation and efficient decomposition techniques.

\section{Related Work}
The study of generative models within a self-consuming loop has garnered significant attention recently. Current works primarily analyze this phenomenon from an empirical perspective. \citet{shumailov2023curse} observe a degeneration of diversity for variational autoencoders and Gaussian mixture models when a portion of model outputs are recursively reused as inputs. Similarly, \citet{briesch2023large} examines the behavior of language models trained in a self-consuming loop from scratch, finding that while quality and diversity improve over initial generations, their output inevitably becomes less diverse after successive training iterations. Advocating the integration of real data, \citet{alemohammad2023self} performs experiments under various mixed training scenarios, concluding that injecting real data can mitigate model collapse. \citet{martinez2023combining} and \citet{martinez2023towards} further demonstrate that training generative models on web-scale datasets polluted by synthetic samples also corrodes the quality of generated data.

In contrast to the empirical observations, theoretical insights into self-consuming loops remain sparse. Both \citet{shumailov2023curse} and \citet{alemohammad2023self} analyze a simple Gaussian toy model to provide theoretical intuition. \citet{bertrand2023stability} establishes that the incorporation of real data can enhance stability within self-consuming loops under assumptions of infinite sample sizes and negligible initial model approximation error. Furthermore, they demonstrate, through direct assumptions on the upper bounds of statistical and optimization errors of generative models, the feasibility of achieving stability even with finite sample sizes. Comparatively, our work establishes theoretical bounds on the distributional discrepancy between synthetic and real-world data without relying on assumptions regarding the bounds of statistical and optimization errors.

\begin{figure}[t]
\vskip 0.2in
\begin{center}
\centerline{\includegraphics[width=\columnwidth]{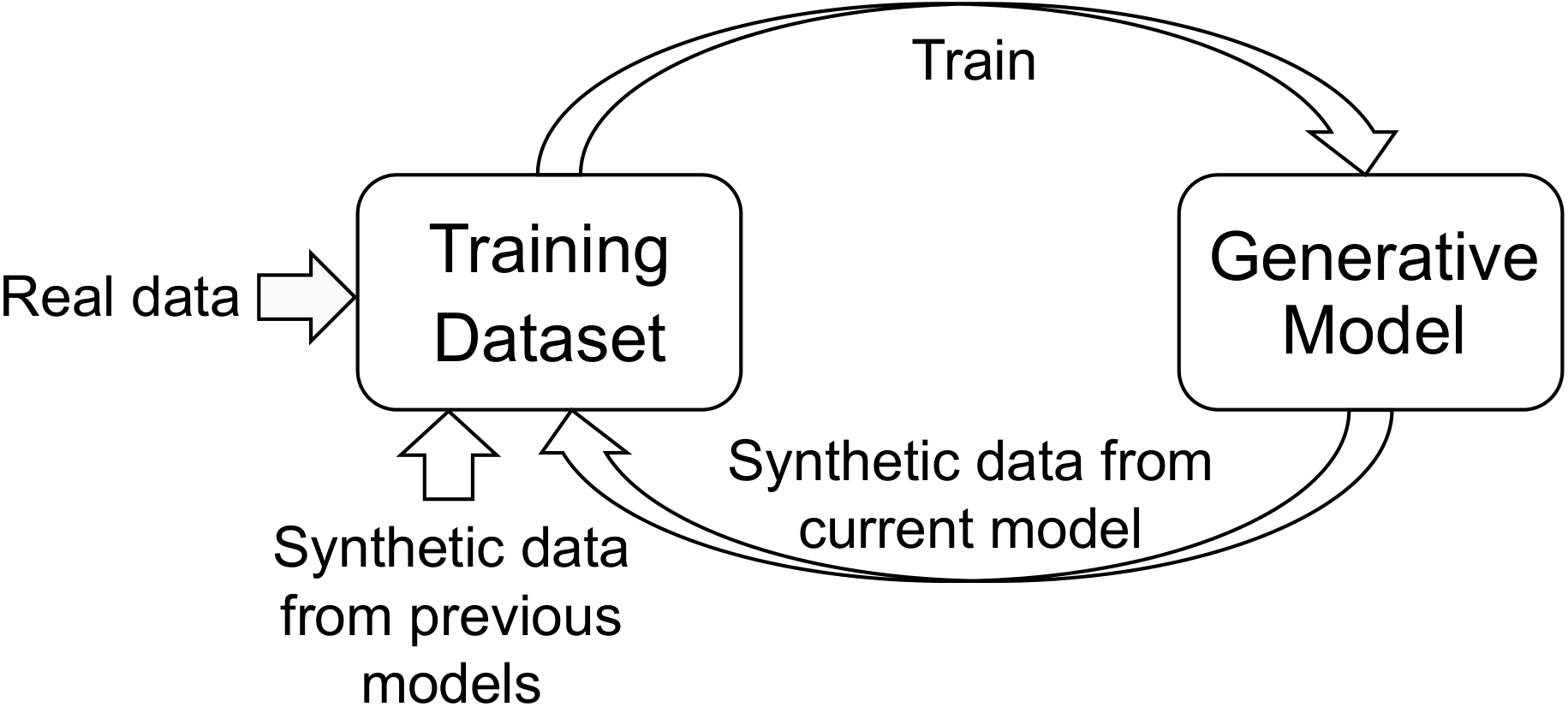}}
\caption{Self-consuming training loop: The initial model $\mathcal{G}_0$ is trained on real dataset $\mathcal{D}_0$. In each subsequent generation $i>0$, $\mathcal{G}_i$ generates samples combined with previous synthetic data and real data into a new training dataset $D_i$. The new model $\mathcal{G}_{i+1}$ is then trained from scratch using $D_i$ for the next generation $i+1$. This repeats until the maximum generation is reached.
}
\label{icml-historical}
\end{center}
\vskip -0.2in
\end{figure}

\section{Background}
\subsection{Self-Consuming Training of Generative Models}
Generative models have advanced in synthesizing realistic data including imagery and text. The resulting synthetic data is widely distributed online and often indistinguishable from genuine content. As generative models evolve, datasets for model training would unintentionally \cite{schuhmann2022laion} or intentionally \cite{huang2022large} include growing proportions of synthetic data alongside real-world samples. The resulting models, in return, create new content, leading to a cycle where successive generations train on datasets with increasingly synthetic proportions. This is termed as a self-consuming training loop, illustrated in Figure \ref{icml-historical}.

\subsubsection{Model Evolution in Self-Consuming Loop}
This paper explores a stochastic process with sequential generations. At generation $i$, we possess a training dataset of $n_i$ samples, \(\{x_i^j\}_{j=1}^{n_i}\), independently and identically distributed, drawn from the distribution \(p_i\). We denote the empirical distribution inferred from this dataset as $\overline{p}_i$. Notably, $p_0$ symbolizes the original distribution w.r.t. real data. Going from generation $i$ to generation $i+1$, our objective is to estimate distribution $p_i$ using samples \(\{x_i^j\}_{j=1}^{n_i}\) via generative model, through parametric models $F_{\theta_{i+1}}: \overline{p}_i \rightarrow p_{\theta_{i+1}}$ or non-parametric estimators $\widehat{F}_{i+1}: \overline{p}_i \rightarrow \widehat{p}_{i+1}$. Here $p_{\theta_{i+1}}$ indicates a generative model parameterized by $\theta_{i+1}$, approximating $p_i$. For generation $i+1$, we resample training data from $p_{i+1} = \sum_{k=1}^{i+1} \beta_{i+1}^k p_{\theta_k} + \alpha_{i+1}p_0$ for parametric models or $p_{i+1} = \sum_{k=1}^{i+1} \beta_{i+1}^k \widehat{p}_{k} + \alpha_{i+1}p_0$ for non-parametric models, with non-negative parameters \(\alpha_{i+1}\) and \(\{\beta_{i+1}^k\}_{k=1}^{i+1}\) summing up to 1. The dataset, sourced from \(p_{i+1}\), comprises a mixture of original data at \(\alpha_{i+1}\) proportion, data generated by previous generations at \(\{\beta_{i+1}^k\}_{k=1}^{i}\) proportions, and current model data at $\beta_{i+1}^{i+1}$ proportion.

Furthermore, the initial generative model is trained on the original dataset from $p_0$. Subsequently, each next generation of models is trained from scratch on a new dataset drawn from the mixed distribution $p_i$. This self-consuming loop repeats until reaching the maximum generation, as shown in Algorithm \ref{alg:example} in the appendix. We then define total variation distance, a metric used to compare probability distributions.
%\begin{algorithm}[tb]
%   \caption{Self-Consuming Loop of Generative Models}
%   \label{alg:example}
%\begin{algorithmic}
%   \STATE {\bfseries Input:} generative model $\mathcal{G}$, proportional coefficients $\{\alpha_i\}_{i=1}^{I}$, $\{\{\beta^k_{i}\}_{k=1}^i\}_{i=1}^I$.
%   \STATE \textbf{Initialize} $\mathcal{D}_0=\{x_i^j\}_{j=1}^{n_0}$, with $x_i^j \sim p_0$, $p_{\theta_1}=\mathcal{G}_0(\mathcal{D}_0)$.
%   \FOR{$i=1$ {\bfseries to} $I$}
%    \STATE \(p_{i}=\beta_{i}^1 p_{\theta_1} + \cdots + \beta_{i}^{i} p_{\theta_{i}} + \alpha_{i}p_0\)
%    \STATE $\mathcal{D}_i=\{x_i^j\}_{j=1}^{n_i}$, with $x_i^j \sim p_i$
%    \STATE $p_{\theta_{i}}=\mathcal{G}_i(\mathcal{D}_i)$
%   \ENDFOR
% \STATE \textbf{Return} %$p_{\theta_{T+1}}$
%\end{algorithmic}
%\end{algorithm}

\begin{definition}[Total Variation Distance] Given two probability distributions \( p \) and \( q \) over a multidimensional space \( \mathbb{R}^d \), the Total Variation Distance between \( p \) and \( q \) is:
\[ TV(p, q) = \frac{1}{2} \int_{\mathbb{R}^d} |p(x) - q(x)| \, dx. \]
%In the discrete setting, where the distributions are defined over a countable set \( X \), the Total Variation Distance is given by:
%\[ TV(p, q) = \frac{1}{2} \sum_{x \in X} |p(x) - q(x)|. \]
\end{definition}

\subsubsection{Data Cycles}
We introduce three different ways of mixing real and synthetic data in self-consuming training loops.

\textbf{General Data Cycle}. In this scenario, each model $\mathcal{G}_i$ (for $i\geq 1$) is trained on a mixture of real data and synthetic data from previous models $\{\mathcal{G}_j\}_{j=0}^{i-1}$. The training distribution $p_i$ is a weighted mixture of the original real data distribution $p_0$ and the synthetic data distributions $\{p_{\theta_j}\}_{j=1}^i$ from previous models, expressed as $p_{i} = \sum_{j=1}^{i} \beta_{i}^j p_{\theta_j} + \alpha_{i}p_0$. Furthermore, when utilizing non-parametric estimation, the training distribution $p_i$ is a weighted mixture of $p_0$ and the non-parametric synthetic distributions $\{\widehat{p}_j\}_{j=1}^i$, expressed as $p_{i} = \sum_{j=1}^{i} \beta_{i}^j \widehat{p}_j + \alpha_{i}p_0$.

\textbf{Full Synthetic Data Cycle}. This extreme case explores the use of training datasets that entirely comprise synthetic data recursively generated by the latest model, without real data. At generation $i$, the training distribution is defined either as the synthetic distribution $p_i=p_{\theta_i}$ or the non-parametric synthetic distribution $p_i=\widehat{p}_i$.

\textbf{Balanced Data Cycle}. This scenario blends real data from the original distribution and synthetic data from all previous models into a training distribution with equal contributions. This results in $p_i$ expressed as $\frac{1}{i+1} (p_0 + p_{\theta_1} + \cdots + p_{\theta_i})$ or $\frac{1}{i+1} (p_0 + \widehat{p}_{1} + \cdots + \widehat{p}_{i})$.

\subsection{Background on Diffusion models}
\textbf{Forward and Reverse processes.} Given a dataset \(\mathcal{D}_x = \{x_i^j\}_{j=1}^{n_i} \subseteq \mathbb{R}^d\) with sample $x_i^j \stackrel{\text{i.i.d.}}{\sim} p_i(x)$, where \(p_i\) denotes the target mixed distribution in the $i$-th generation, the forward diffusion process is:
\begin{align}
    dx_i = f(x_i,t)dt + g(t)dw_t, \quad x_i(0) \sim p_0, \notag
\end{align}
where \( w_t \) is the standard Wiener process, \( f(\cdot, t) : \mathbb{R}^d \rightarrow \mathbb{R}^d \) is the \textit{drift coefficient}, and \( g(\cdot) : \mathbb{R} \rightarrow \mathbb{R} \) is the \textit{diffusion coefficient}. The reverse process generates \(x_i(0) \sim p_i\) from \(x_i(T) \sim p_{i,T}\). The reverse-time SDE \cite{anderson1982reverse} is: 
\begin{align}
    dx_i = \left[ f(x_i,t) - g(t)^2 \nabla_{x_i} \log p_{i,t}(x_i) \right] dt + g(t) d\bar{w}_t, \notag
\end{align}
where $\bar{w}_t$ is a Wiener process from $ T $ to $ 0 $, starting with $ p_{i,T} \approx \pi $, with $\pi$ a known prior such as Gaussian noise.

\textbf{Loss objectives.} The task is to estimate the unknown Stein score function $\nabla_{x_i} \log p_{i,t}(x_i)$ by minimizing weighted denoising score matching objectives:
\begin{align}
    &\mathcal{L}(\theta; \lambda(\cdot)) := \mathbb{E}_{t\sim \mathcal{U}(0,T)} \left[ \lambda(t) \cdot \mathbb{E}_{x_i(0)\sim p_{i,0}} \left[ \mathbb{E}_{x_i(t)\sim p_{i,t|0}} \left[ \right.\right.\right. \notag\\
    &\left.\left.\left.\|s_{t,\theta}(x_i(t)) - \nabla_{x_i(t)} \log p_{i,t|0}(x_i(t)|x_i(0))\|_2^2 \right] \right] \right], \label{loss}
\end{align}
where $\theta^*:=\arg\min_{\theta}\mathcal{L}(\theta;\lambda(\cdot))$ and $\mathcal{U}(0,T)$ denotes the uniform distribution. 
 $\lambda(t):[0,T]\to\mathbb{R}_+$ is a weighting function. The score function $s_{t,\theta}:\mathbb{R}^d\to\mathbb{R}^d$ can be parameterized as a neural network. The time-dependent score-matching loss is:
 \begin{align}
     &\widetilde{\mathcal{L}}(\theta;\lambda(\cdot))\notag:=\mathbb{E}_{t\sim\mathcal{U}(0,T)}\left[\lambda(t)\right.\notag \\ &\left.\mathbb{E}_{x_i(t)\sim p_{i,t}}\left[\|s_{t,\theta}(x_i(t))-\nabla_{x_i(t)}\log p_{i,t}(x_i(t))\|_2^2\right]\right]. \notag
 \end{align}
 
 \textbf{Training}. Let $\hat{\mathcal{L}}_{n_i}$ denote the Monte-Carlo estimation of $\mathcal{L}$ defined in Equation \ref{loss} on the training dataset. The gradient flows over the empirical loss and the  population loss are:
\begin{align}
    &\frac d{d\tau}\hat{\theta}_i(\tau)=-\nabla_{\hat{\theta}_i(\tau)}\hat{\mathcal{L}}_{n_i}(\hat{\theta}_i(\tau);\lambda(\cdot)),\quad\hat{\theta}_i(0):=\theta_0,\notag \\
    &  \frac d{d\tau}\theta_i(\tau)=-\nabla_{\theta_i(\tau)}\mathcal{L}_{n_i}(\theta_i(\tau);\lambda(\cdot)),\quad\theta_i(0):=\theta_0. \notag
\end{align}
The learned score functions at training time $\tau$ for SDE time $t$ are $s_{t,\hat{\theta}_i(\tau)}(x_i(t))$ and $s_{t,\theta_i(\tau)}(x_i(t))$, respectively.

\textbf{Score networks}.The score function 
$s_{t, \theta}(x_i)$ is parameterized using the subsequent random feature model:
$$
\frac{1}{m_i} A\sigma(Wx_i + Ue(t)) = \frac{1}{m_i} \sum_{j=1}^{m_i} a_j\sigma(w_j^\top x_i + u_j^\top e(t)).
$$
Here $\sigma$ is ReLU activation, \(A = (a_1, \ldots, a_{m_i}) \in \mathbb{R}^{d \times m_i}\) is trainable, \(W = (w_1, \ldots, w_{m_i})^\top \in \mathbb{R}^{m_i \times d}\) and \(U = (u_1, \ldots, u_{m_i})^\top \in \mathbb{R}^{m_i \times d_e}\) are initially randomized and remain fixed during training. The function \(e : \mathbb{R}_{\geq 0} \to \mathbb{R}^{d_e}\) embeds time.  Assuming 
$a_j$, $w_j$ and $u_j$ are i.i.d. from a distribution $\rho$, as $m_i\to\infty$, this approaches:
$$
\bar{s}_{t,\bar{\theta}}(x)= \mathbb{E}_{(w,u)\sim \rho_0} \left[ a(w, u)\sigma(w^\top x + u^\top e(t)) \right],
$$
where \(a(w, u)\) is \(\frac{1}{\rho_0(w,u)} \int_{\mathbb{R}^d} a \rho(a, w, u) da \) and \( \rho_0(w, u)\) is \(\int_{\mathbb{R}^d} \rho(a, w, u) da \). By the positive homogeneity of ReLU, we assume \( \|w\| + \|u\| \leq 1 \). The optimal solution is denoted by $\bar{\theta}^*$ when replacing $s_{t,\theta}(x)$ in loss objectives with $\bar{s}_{t,\bar{\theta}}(x)$.

The kernel \( k_{\rho_0} (x, x')\) is defined as \(\mathbb{E}_{(w,u)\sim p_0} [\sigma(w^\top x + u^\top e(t))\sigma(w^\top x' + u^\top e(t))] \), and denoted by \( \mathcal{H}_{k_{\rho_0}} \), the induced reproducing kernel Hilbert space (RKHS). It follows that \( \bar{s}_{t,\bar{\theta}} \in \mathcal{H}_{k_{\rho_0}} \) if the RKHS norm \( \|\bar{s}_{t,\bar{\theta}}\|_{\mathcal{H}_{k_{\rho_0}}}^2 = \mathbb{E}_{(w,u)\sim \rho_0} [\|a(w, u)\|_2^2] \) is finite. The discrete version is the empirical average \( \|s_{t,\theta}\|_{\mathcal{H}_{k_{\rho_0}}}^2 = \frac{1}{m} \sum_{j=1}^{m} \|a(w_j, u_j)\|_2^2 \).

\section{Theoretical Results for Diffusion Models within Self-Consuming Loops}\label{applica}
In this section, we apply our theoretical framework to diffusion models with a one-hidden-layer neural network score function. The diffusion models have recently gained significant popularity due to their outstanding performance across various applications. Notable examples include DALL·E \cite{ramesh2022hierarchical} and Stable Diffusion \cite{rombach2022high}, both demonstrating notable advancements. However, the curse of dimensionality renders it challenging to obtain upper bounds on the statistical errors and optimization errors that compound over successive training generations. Inspired by \citet{yang2022mathematical,yang2022generalization,li2023generalization}, we conduct a more fine-grained analysis of this error propagation in diffusion models in the self-consuming loops. This enables us to obtain an upper bound on the TV distance between the synthetic data distributions produced by future models and the original real data distribution under various mixed training scenarios.%that avoids the curse of dimensionality when early-stopped.

\subsection{General Data Cycle: Flexibly Regulating Data Composition Mixture}

This framework allows us to flexibly control the compositional mixture between real and synthetic data in the training distribution $p_i$. Our theoretical analysis quantifies the impacts of this adaptive training approach with mixed data on the cumulative error and the fidelity of models in self-consuming loops.

\begin{theorem}\label{diffusion model}
Suppose that \( p_i \) is continuously differentiable and has a \textit{compact support set}, i.e., \( \|x\|_{\infty} \) is uniformly bounded, and there exists a RKHS \( \mathcal{H}_{k_{p_0}} \) such that \( \bar{s}_{0,\bar{\theta}^*} =  \mathbb{E}_{(w,u)\sim \rho_0} \left[ a^*(w, u)\sigma(w^\top x + u^\top e(0)) \right] \in \mathcal{H}_{k_{p_0}} \). Suppose that the initial loss, trainable parameters, the embedding function \( e(t) \) and weighting function \( \lambda(t) \) are all bounded. Let $n_{i}$ be the number of training samples obtained from the distribution $p_{i} = \sum_{j=1}^{i} \beta_{i}^j p_{\theta_j} + \alpha_{i}p_0$. Choose $m_i\asymp n_i$ and $\tau_{i+1}\asymp \sqrt{n_i}$\footnote{We denote $B\asymp \widetilde{B}$ if there are absolute constants $c_1$ and $c_2$ such that $c_1B\leq \widetilde{B}\leq c_2B$.}. Then,  with probability at least \( 1 - \delta \),
\begin{align}
    &TV(p_{\theta_{i+1}},p_0)\notag \\
    &\lesssim \sum_{k=0}^i A_{i-k} \left(n_{i-k}^{-\frac{1}{4}}\sqrt{d\log \frac{di}{\delta}}+\sqrt{KL(p_{i-k,T}\|\pi)}\right),\notag
\end{align}
   where $A_i=1, A_{i-k}=\sum_{j=i-k+1}^i \beta_j^{i-k+1} A_j$ for $1\leq k \leq i$ and $\lesssim$ hides universal positive constants that depend solely on $T$.
\end{theorem} 
\textbf{Proof sketch of Theorem \ref{diffusion model}.} We begin by leveraging the triangle inequality to decompose the TV distance between $p_{\theta_{i+1}}$ and $p_0$ as follows:
$$
TV(p_{\theta_{i+1}},p_0)\leq TV(p_{\theta_{i+1}},p_i)+TV(p_i,p_0)
$$
Next, we focus on bounding the first term $TV(p_{\theta_{i+1}},p_i)$. Through an application of Pinsker's inequality, we relate this term to the KL divergence between $p_i$ and $p_{\theta_{i+1}}$. Subsequently, according to Theorem 1 in \citet{song2021maximum}, the KL divergence can be upper bounded by the training loss $\widetilde{\mathcal{L}}(\hat{\theta}_{i+1}(\tau_{i+1}))$ up to a small error. To dissect the training loss $\widetilde{\mathcal{L}}(\hat{\theta}_{i+1}(\tau_{i+1}))$, we leverage several decompositions and partition it into multiple constituent terms: $\widetilde{\mathcal{L}}(\hat{\theta}_{i+1}(\tau_{i+1})) - \widetilde{\mathcal{L}}(\theta_{i+1}(\tau_{i+1}))$, $\widetilde{\mathcal{L}}(\theta_{i+1}(\tau_{i+1}))-\bar{\tilde{\mathcal{L}}}(\bar{\theta}_{i+1}(\tau_{i+1}))$, and $\bar{\tilde{\mathcal{L}}}(\bar{\theta}_{i+1}(\tau_{i+1}))$.

For the first term, $\widetilde{\mathcal{L}}(\hat{\theta}_{i+1}(\tau_{i+1})) - \widetilde{\mathcal{L}}(\theta_{i+1}(\tau_{i+1}))$, we bound the gap through general norm estimates of parameters trained under gradient flow dynamics, as delineated in \citet{li2023generalization}, in conjunction with classical analyses leveraging Rademacher complexity. For the second term, $\widetilde{\mathcal{L}}(\theta_{i+1}(\tau_{i+1}))-\bar{\tilde{\mathcal{L}}}(\bar{\theta}_{i+1}(\tau_{i+1}))$, we derive the bound via gradient flow analysis and properties of the RKHS norm. Regarding the third term, $\bar{\tilde{\mathcal{L}}}(\bar{\theta}_{i+1}(\tau_{i+1}))$, we establish the bound through properties of the score function space.

By analyzing each constituent term and selecting the model width $m_i \asymp n_i$, optimal early-stopping training time $\tau_{i+1} \asymp \sqrt{n_i}$, in accordance with properties of the underlying data distribution, we arrive at upper bounds for $TV(p_{\theta_{i+1}},p_i)$. Finally, we substitute this bound back into the original decomposition to obtain the overall bound on $TV(p_{\theta_{i+1}},p_0)$. We then derive the final result through recursive solving.

\begin{remark}\textbf{Comparision with Previous Works.} In the field of self-consuming generative modeling, \citet{shumailov2023curse} and \citet{alemohammad2023self} provided foundational insights through their analysis of a simplistic multivariate Gaussian toy model. However, their approach prioritized providing theoretical intuition instead of a detailed and rigorous theoretical analysis. A more pertinent comparison to our research can be drawn with Theorem 2 in \citet{bertrand2023stability}.

\citet{bertrand2023stability} made significant strides by establishing an upper bound on the deviation of generative model output parameters from the optimal parameters, denoted as $\|\theta_{i+1}-\theta^*\|$. This was achieved under specific assumptions on the upper bounds of statistical and optimization errors in generative models, as outlined in their Assumption 3. 

Our research diverges from \citet{bertrand2023stability}'s work in several key aspects: (1) We eschew the direct assumption of upper bound constraints on statistical and optimization errors. Instead, our analysis delves deeply into the nature of statistical errors by employing kernel density estimation theory and concentration inequalities for a nuanced understanding. (2) We specifically examine optimization errors in generative models, particularly focusing on diffusion models as elaborated in Section \ref{applica}. This approach allows us to circumvent the need for direct assumptions on errors. (3) Our analysis extends beyond the scope of \citet{bertrand2023stability}. While they focused on the parameter difference 
$\|\theta_{i+1}-\theta^*\|$, our research tackles the more complex task of assessing the distributional discrepancy between synthetic and real data, represented as $TV(p_{\theta_{i+1}}, p_0)$, which is inherently more challenging since the optimal estimator does not perfectly mirror the original real data distribution $p_0$. (4) Finally, our assumptions regarding the composition of the training set are more general and practical compared to those of \citet{bertrand2023stability}. We consider a training set that incorporates data from all previous generations alongside real data, offering a broader and more realistic foundation for training the generative model. This stands in contrast to their assumption, which limits the training set to real data and synthetic data from only the most recent generation.

In summary, our work extends beyond existing literature by overcoming assumptions on statistical and optimization error bounds, assessing distributional discrepancy between synthetic and real data, and employing more realistic assumptions about the training dataset.
\end{remark}

\begin{remark}\textbf{Optimal Early-Stopping Strategy}.
When we select $m_i \asymp n_i$ and omit the $\sqrt{d\log (d/\delta)}$ term, the bound for $TV(p_{\theta_{i+1}},p_i)$ in Theorem \ref{diffusion model} can be expressed as:
$$
\sum_{k=0}^i A_{i-k} \left(\frac{\tau_{i+1-k}^{3/2}}{n_{i-k}}+\frac{1}{\sqrt{\tau_{i+1-k}}}+\sqrt{KL(p_{i-k,T}\|\pi)}\right).
$$
By optimally choosing the early-stopping time $\tau_{i+1} \asymp \sqrt{n_i}$, we arrive at the final result. In essence, our analysis provides that through early stopping, the TV distance can be controlled when retraining diffusion models on mixed datasets over successive generations. The key insight is that by early-stopping the training at the optimal time, we prevent overfitting to the training distribution, thereby controlling the discrepancy between the synthetic data distribution and the original data distribution. 
\end{remark}

\begin{remark}\label{diffu real} \textbf{Real Data Integration Across Generations.} %Unlike Remark \ref{remark 4.9}, which explores the effect of incorporating real data only at the final generation to control errors, t
Previous experimental results suggest that when the proportion of real data is sufficiently large, the error can be effectively controlled \cite{alemohammad2023self, shumailov2023curse}. This remark examines the impact of incorporating real data at each generation on the theoretical outcomes for diffusion models. Specifically, we assume $p_j=\alpha p_0+(1-\alpha)p_{\theta_j}$ for $1\leq j \leq i$ and $0<\alpha<1$. Additionally, we assume that the training set sizes and $KL$ terms are of
the same order of magnitude across all generations. This setup allows us to derive the following bound:
\begin{align}
 &TV(p_{\theta_{i+1}},p_0) \lesssim \notag \\
   &
    \left(1-(1-\alpha)^{i+1}\right)\alpha^{-1} (n^{-\frac{1}{4}}\sqrt{d\log \frac{di}{\delta}}+\sqrt{KL(p_{i,T}\|\pi)}).\notag
\end{align}
Based on our theoretical results, we observe that to control the error, the requirement on the proportion of real data added at each generation is more relaxed compared to the approach discussed in Remark \ref{remark 4.9} in the following subsection, where real data is only added in the final generation. Specifically, while Remark \ref{remark 4.9} requires the proportion of real data to increase with the number of training generations, i.e., $\alpha=\Omega(\frac{i-1}{i})$, this analysis shows that if real data is added at each generation, a suitable constant value for $\alpha$ is sufficient to effectively control the error. Additionally, due to the presence of term $\alpha^{-1}$, the value of $\alpha$ should not be too small, as it may lead to difficulties in controlling the error. For instance, as $\alpha$ tends to 0, employing Taylor expansion, we observe that the error accumulates linearly with the increase of generation $i$.

\end{remark}

\subsection{Full Synthetic Data Cycle: Error Control with Quartically Sampling or High Real Data Ratio}
Here we thoroughly analyze the fully synthetic data cycle, wherein each model is trained using synthetic data from the most recent generative model. Although in practical data collection processes, new datasets typically retain portions of original real data even in successive generations, analyzing this extreme synthetic-only training loop provides valuable theoretical insights. Specifically, our theoretical analysis shows that the error can be controlled by increasing the sampling quantity quartically across generations or by injecting sufficient real data in the final generation.
\begin{corollary}[Worst Case]\label{full synthetic data cycle} %Suppose $TV(p_{\theta_{i+1}},\widehat{p}_i)=\epsilon_{n_i}$ for all $i$. 
Let $n_{i}$ represent the number of training samples obtained from the distribution $p_{i}$ at the $i$-th generation. We define $p_i$ as $p_i=p_{\theta_i}$. Then, with probability at least $1-\delta$,
\begin{align}
&TV(p_{\theta_{i+1}},p_0)\notag \\
&\lesssim \sum_{k=1}^i\left(n_{k}^{-\frac{1}{4}}\sqrt{d\log \frac{di}{\delta}}+\sqrt{KL(p_{k,T}\|\pi)}\right). \notag
\end{align}
\end{corollary}
\begin{remark}\label{remark full 1}\textbf{Controlling Error Accumulation with Quartic Sample Growth.} %To better elucidate the behavior of generative models in the extreme full synthetic case within self-consuming training loops, without loss of generality, let us first posit that $\epsilon_{n_j}=\mathcal{O}(n_j^{-s/(2s+2d)})$ for $0\leq j \leq i$. 
According to the classical results in \citet{van2014probability}, since \(\pi\) (e.g., the Gaussian density) is log-Sobolev, \(KL(p_{i, T} || \pi)\) is exponentially small in \(T\). To better elucidate the behavior of models in the extreme full synthetic case, we first posit that \(KL(p_{k, T} || \pi) = \mathcal{O}(\epsilon^2/i^2)\) for \(1\leq k \leq i\). Consequently, the result of Corollary \ref{full synthetic data cycle} indicates that the number of training samples \(n_i\) must grow quartically, specifically as \(n_k = \Omega\left((i\sqrt{d}/\epsilon)^{4}\right)\) for $1\leq k \leq i$, in order to constrain the TV distance to \(\mathcal{O}(\epsilon)\).

Intuitively, in the absence of any grounding from real data, errors can accumulate rapidly as the model at each generation is trained only on synthetic samples generated by the model from the latest generation. Without adequate samples to properly approximate the training distribution, the statistical error and the inherent sample bias accumulate.

To counteract this, each successive generation necessitates an increasingly expansive training set to provide sufficient coverage of the distribution. Our analysis quantifies this requirement, showing a quartic growth in the training samples is imperative to control the error.
\end{remark}
\begin{remark}\label{remark 4.9}
\textbf{Integrating Sufficient Real Data in the Final Generation for Error Control}. In addition to the discussion in Remark \ref{remark full 1}, that involves incrementally raising the sample size through quartic growth across generations, an alternative approach is to increase the proportion of original data in the final generation of training to control the error. To theoretically investigate the role of the real data from the original distribution, we assume $p_i=\alpha p_0+(1-\alpha)p_{\theta_i}$, $p_j=p_{\theta_j}$ for $1\leq j \leq i-1$ and \(KL(p_{j, T} || \pi) = \mathcal{O}(\epsilon^2/i^2)\) for all $j$. %, and further posit that $\epsilon_{n_j}=\mathcal{O}(n_j^{-s/(2s+2d)})$. 
Additionally, we assume that the training set sizes are of the same order of magnitude across all generations, specifically, $n_j=\mathcal{O}(d^2/\epsilon^{4})$ for $0 \leq j \leq i$. Under these assumptions, we obtain the bound:
$$
TV(p_{\theta_{i+1}},p_0)=\mathcal{O}((1-\alpha)i\epsilon).
$$
In particular, to ensure $TV(p_{\theta_{i+1}},p_0)=\mathcal{O}(\epsilon)$, it suffices to have the proportion $\alpha=\Omega(\frac{i-1}{i})$. 
This analysis reveals that as the number of generations increases, the proportion $\alpha$ of data from the original distribution must be progressively augmented to constrain the error to $\mathcal{O}(\epsilon)$. This suggests that later generational models have greater compounding drift from the original distribution, necessitating a larger fraction of real samples for grounding. 
By incorporating sufficient proportions of real data in the final generation of training, we can then control the error and avoid the requirement of quartic sample growth.

%Additionally, we discuss in the Remark \ref{diffu real} the TV distance bound of the diffusion model within a self-consuming loop when real data is added at each generation.
\end{remark}

\subsection{Balanced Data Cycle: Optimizing Sample Efficiency Across Generations}\label{balanced bound}
In this section, we analyze the scenario of balanced data cycle, wherein the training distribution comprises a uniform mixture of the original data distribution and synthetic data distributions from all previous generative models. Moreover, we investigate the sample complexity $\{n_j\}_{j=0}^i$ required to control the error over generations. Interestingly, our analysis demonstrates that the requisite number of samples progressively decreases as the number of generations grows.
\begin{corollary}\label{balanced data corollary }%Suppose $TV(p_{\theta_{i+1}},\widehat{p}_i)=\epsilon_{n_i}$ for all $i$. 
Define $p_i=\frac{1}{i+1}(p_0+p_{\theta_1}+p_{\theta_2}+\cdots+p_{\theta_i})$ for $i \geq 1$. Let $n_i$ be the number of training samples obtained from the distribution $p_{i}$ at the $i$-th generation. With probability at least $1-\delta$,
\begin{align}
 &TV(p_{\theta_{i+1}},p_0) \lesssim n_{i}^{-\frac{1}{4}}\sqrt{d\log \frac{di}{\delta}}+\sqrt{KL(p_{i,T}\|\pi)}+\notag \\
 &\sum_{k=0}^{i-1}\sum_{j=k}^{i-1}  \frac{\Gamma(j+2)}{\Gamma(i+2)}\left(n_{k}^{-\frac{1}{4}}\sqrt{d\log \frac{di}{\delta}}+\sqrt{KL(p_{k,T}\|\pi)}\right), \notag
\end{align}
where the Gamma function $\Gamma(j)=(j-1)!$ and $j$ is a positive integer. 
\end{corollary}
\begin{remark}
\textbf{Diminishing Sample Complexity Over Generations}. In the scenario of balanced data cycle, controlling the total accumulation of errors necessitates sufficiently large training sample sizes for the initial generations of generative models. Intriguingly, as the number of generations increases, the demand for training samples progressively decreases. %To elucidate this result, we posit $\epsilon_{n_t}=\mathcal{O}(n_t^{-s/(2s+2d)})$ for $0\leq t\leq i$.
We suppose \(KL(p_{k, T} || \pi) = \mathcal{O}(\epsilon^2/i^2)\) for all \(k\).
By selecting $n_k=\mathcal{O}\left((\frac{(i+1)\sqrt{d}}{\epsilon}\sum_{j=k}^{i-1}\frac{\Gamma(j+2)}{\Gamma(i+2)})^{4}\right)$ for $0 \leq k \leq i-1$ and $n_i=\mathcal{O}\left(d^2/\epsilon^{4}\right)$, we obtain:
$$
TV(p_{\theta_{i+1}},p_0)=\mathcal{O}(\epsilon).
$$
Critically, the required training samples gradually decrease with increasing generations. In particular, the number of training samples required for the first generation is 
$n_0=\mathcal{O}\left((\frac{(i+1)\sqrt{d}}{\epsilon}\sum_{j=0}^{i-1}\frac{\Gamma(j+2)}{\Gamma(i+2)})^{4}\right)$, while for the 
$i$-th generation, the required training samples are $n_i=\mathcal{O}\left((\frac{\sqrt{d}}{\epsilon})^{4}\right)$.

Furthermore, this finding aligns with intuition, as data from earlier generative models is incorporated into training sets earlier, thereby influencing more generations of models and having a greater impact on the cumulative error across the entire process. Therefore, to ensure the fidelity of the initially generated data, more training data is needed to constrain the error.
\end{remark}

\subsection{Phase Transition in Error Dynamics With Synthetic Data Augmentation}
While numerous studies have empirically demonstrated the benefit of synthetic data in enhancing model performance \cite{azizi2023synthetic,burg2023data}, the established error bounds in Section \ref{diffusion model} to \ref{balanced bound} for various mixed training scenarios instead raise skepticism regarding whether synthetic data could potentially impair model performance during self-consuming training. The reality appears more nuanced than a simple binary categorization. In this section, we conduct a theoretical analysis by keeping the real data fixed and flexibly varying the numbers of synthetic samples. Interestingly, in the absence of sampling bias, we identify a regime where modest amounts of synthetic data can degrade performance. However, beyond a certain threshold, increasing synthetic data improves model performance.
\begin{corollary}\label{phase transition} %Suppose $TV(p_{\theta_{i+1}},\widehat{p}_i)=\epsilon_{n+m}$ for all $i$. 
Define $p_i=\frac{n}{n+m}p_0+\frac{m}{n+m}p_{\theta_i}$ for $i\geq 1$. Suppose that the $KL$ terms are of the same order of magnitude across all generations. Let $n+m$ be the number of training samples obtained from the distribution $p_{i}$ at the $i$-th generation. With probability at least $1-\delta$,
\begin{align}
&TV(p_{\theta_{i+1}},p_0)\lesssim \left(1+\frac{m}{n}\right)
  \left(1-(\frac{m}{n+m})^{i+1}\right)\notag\\
  &\times\left((n+m)^{-\frac{1}{4}}\sqrt{d\log \frac{di}{\delta}}+\sqrt{KL(p_{i,T}\|\pi)}\right). \notag
\end{align}
\end{corollary}

\begin{remark}\label{unbias} \textbf{The Impact of Increasing Synthetic Data on Error Propagation}.
There is a tradeoff when fixing the number of training samples $n$ from the original distribution and increasing the number of synthetic samples $m$ generated by the model. On the one hand, augmenting the training set with additional synthetic samples $m$ can reduce the statistical error and estimation error for each generation. This approach is intuitive, as synthetic data transfers knowledge from previously used real data to subsequent generations, thereby enlarging the effective size of the dataset. 
On the other hand, as $m$ increases, the proportion of original training samples in the mixture distribution $p_i = (n/(n+m))p_0 + (m/(n+m))p_{\theta_i}$ decreases. This exacerbates distribution shift from the original data distribution $p_0$, accumulating errors over successive generations. 

Therefore, our results imply a tradeoff between reducing statistical and estimation errors for each generation by adding more synthetic data, and controlling the cumulative effects of distribution shift over successive generations by reducing the proportion of synthetic data. This trade-off is also recognized as the phase transition as evidenced by the experimental findings in \citet{alemohammad2023self}. To further investigate this phenomenon, we posit that \(m = \lambda n\) where \(\lambda > 0\) and disregard the \(KL\) term, as it is exponentially small in \(T\).
 Under these assumptions, the result described in Corollary \ref{phase transition} is as follows:
 \begin{align}
TV(p_{\theta_{i+1}},p_0)=\mathcal{O}\left(\frac{1}{n^{1/4}}\frac{(1+\lambda)^{i+1}-\lambda^{i+1}}{(1+\lambda)^{i+\frac{1}{4}}}\right). \notag
\end{align}
Our focus is on the impact of synthetic data on the error, which is implicated by the effect of \(\lambda\). Consider $f(\lambda,i)=((1+\lambda)^{i+1}-\lambda^{i+1})/(1+\lambda)^{i+\frac{1}{4}}$. Clearly, as \(\lambda\) approaches infinity, by the Taylor's Formula, \(f\) converges to $(i+1)(1+\lambda)^{-1/4}$, meaning that 
$f$ approaches 0. This implies that as we incorporate infinitely abundant synthetic data, the TV distance becomes significantly small.

Notably, in contrast to the experimental results in \citet{alemohammad2023self}, where errors continuously decrease with increased $\lambda$ under unbiased sampling beyond Gaussian modeling, our theory indicates that $f$ experiences an initial increase with a rise in $\lambda$ but subsequently decreases after a phase transition point. Crucially, there is no analytical solution expressing the critical value \(\lambda'\) at this phase transition point in terms of the number of training generations $i$. Nevertheless, numerical analysis reveals that \(\lambda'\) is positively correlated with \(i\), increasing as \(i\) does. 

This aligns with intuition, as the cumulative effects of distribution shift compound over successive generations with larger $i$. To counteract the exacerbated drift from the original distribution for later generative models, it becomes imperative to substantially reduce the statistical and estimation errors for each individual generation. This necessitates further expanding the training set by incorporating more synthetic data, which transfers knowledge from previously utilized real data to subsequent generations. In this way, the phase transition point $\lambda'$ where errors decrease given fixed real data, shifts larger as generations continue.
\end{remark}

\section{Theoretical Results for KDE within Self-Consuming Loops}
%When dealing with continuous distributions, utilizing the standard empirical measure to approximate the original true distribution is inadequate, as the total variation distance between any continuous density $p$ and any atomic measure (such as the standard empirical measure) is 1. This is due to the fundamental mismatch between a continuous distribution and a discrete approximation. Instead, another type of empirical measure, namely the kernel density estimate, is widely used. Unlike the standard empirical measure which places discrete probability masses on the observed data points, kernel density estimation convolves the data points with a kernel function to produce a continuous probability density function. This continuous estimate can better capture the characteristics of an underlying continuous distribution.

This section presents our theoretical analyses of self-consuming non-parametric generative models. Specifically, we derive TV distance bounds between synthetic data distributions produced by future models and the original real data distribution under various mixed training scenarios. We introduce kernel density estimation (KDE), a widely utilized non-parametric density estimation method \cite{devroye2001combinatorial}. Let $x_i^1, \cdots,x_i^{n_i}$ be i.i.d. random variables in $\mathbb{R}^d$ with density $p_i$. Then, the kernel estimate is defined by:
\begin{equation}
\widehat{p}_{i+1}(x) = \frac{1}{n_ih_i^d} \sum_{j=1}^{n_i} K \left( \frac{x - x_i^j}{h_i} \right),
\end{equation}
where $h_i > 0$ is a bandwidth parameter and $K$ is a kernel.

The following standard notations are used. A vector $\alpha=(\alpha_1, \ldots, \alpha_d)$ , composed of nonnegative integers $\alpha_j$ is defined as a multi-index. We denote $|\alpha| = \alpha_1 + \ldots + \alpha_d$, $\alpha! = \alpha_1! \cdots \alpha_d!$ and for a vector $x_i=(x_{i,1}, \ldots, x_{i,d}) \in \mathbb{R}^d$, we define $x_i^{\alpha} = x_{i,1}^{\alpha_1} \cdots x_{i,d}^{\alpha_d}$. The  partial derivative of the function $p_i$ is represented as $\partial^{\alpha}p_i = \partial_1^{\alpha_1} \cdots \partial_d^{\alpha_d}p_i$. Additionally, for $s \geq 0$, $W^{s,1}(\mathbb{R}^d)$ is the Sobolev space of functions $p_i$ whose weak ("distributional") partial derivatives $\partial^{\alpha}p_i$, $|\alpha| \leq s$, are integrable. 

The Sobolev assumption is not only very common in KDE \cite{kroll2021density,cleanthous2019minimax}, but also frequently employed in more complex generative models, such as Transformers \cite{fonseca2023continuous}, GANs \cite{mroueh2017sobolev}, and VAEs \cite{turinici2019x}. 

In the following Lemma \ref{lemma}, we provide a theoretical guarantee on controlling the discrepancy between the kernel estimate $\widehat{p}_{i+1}$ and the true underlying distribution $p_i$ in the $i$-th generation. This serves as a key preliminary step before relating $\widehat{p}_{i+1}$ to the distance original data distribution $p_0$. By bounding the gap between $\widehat{p}_{i+1}$ and $p_i$, the lemma establishes a foundation for controlling the accumulation of statistical errors from finite sample estimates under the self-consuming loop framework. 

\begin{lemma}\label{lemma} Let $s\geq 1$. Suppose $p_i \in W^{s,1}(\mathbb{R}^d)$, and for some $\epsilon>0$, $\int |x|^{d+\epsilon}p_i(x)\ dx<\infty$ for all $i$. Let $n_i$ be the number of training samples obtained from distribution $p_i$ in the $i$-th generation. By appropriately selecting the kernel $K$ as defined in Definition \ref{kernel definition} in the appendix, and setting the bandwidth parameter $h_i=n_i^{-1/(2s+2d)}$, the following holds with probability at least $1-\delta$:
\begin{align}
   &TV(\widehat{p}_{i+1}, p_i)\leq \frac{1}{2}n_i^{-\frac{2s+d}{4s+4d}} (1+\gamma_{n_i})\sqrt{ \int K^2}\int \sqrt{p_i}\notag \\
   &+\frac{1}{2}n_i^{-\frac{s}{2s+2d}}\varphi(s,K,p_i)+n_i^{-\frac{s}{2s+2d}}\sqrt{\frac{1}{2}(\int |K|)^2\log \frac{2}{\delta}}, \notag
\end{align}
where $\gamma_{n_i} \to 0$ as $n_i \to \infty$ and $\varphi(s,K,p_i)$ is a finite function.
\end{lemma}

\textbf{Proof sketch of Lemma \ref{lemma}}. The proof utilizes the kernel density estimation framework \cite{holmstrom1992asymptotic,devroye2001combinatorial} to bound the TV distance between the non-parametric distribution $\widehat{p}_{i+1}$ estimated from finite samples and the true underlying distribution $p_i$. First, McDiarmid's inequality is applied to show that with high probability, the TV distance between $\widehat{p}_{i+1}$ and $p_i$ is bounded by the expected TV distance $\mathbb{E}\int |\widehat{p}_{i+1}-p_i|$ plus a concentration term that decays as $n_i^{-1/2}h_i^{-d}$, where $h_i$ is the kernel bandwidth. Next, the expected TV distance is decomposed into a bias term $\int|p_{i}*K_{h_{i}}-p_{i}|$ representing the distance between $p_i$ convolved with the kernel $K_{h_i}$ and $p_i$, and a variation term $\mathbb{E}\int|\widehat{p}_{i+1}-p_i*K_{h_i}|$. Using Taylor expansion, properties of the kernel, and Young's inequality, this bias term is shown to decay at the rate $h_i^s$, where $s$ depends on the smoothness of $p_i$. As for the variation term, by applying the Schwarz inequality and Carlson’s inequality, this term decays as $(n_ih^d_i)^{-1/2}$. Finally, by optimally selecting the kernel bandwidth $h_i=n_i^{-1/(2s+2d)}$, the bias term, variation term, and concentration term can be balanced, leading to the final TV bound. The key steps involve utilizing kernel density theory and concentration inequalities.

\begin{remark}
    Regarding the existing work related to Lemma \ref{lemma}, \citet{jiang2017uniform} derives $L_\infty$ density estimation bounds for KDE. On the other hand, \citet{kroll2021density} focuses on the adaptive minimax density estimation problem. As for the $L_1$ error, which is the focus of our paper, \citet{cleanthous2019minimax}, \citet{holmstrom1992asymptotic}, and \citet{devroye2001combinatorial} establish upper bounds in expectation. In contrast, Lemma \ref{lemma} utilizes concentration inequalities to derive high-probability finite-sample bounds in the $L_1$ norm.
\end{remark}

\begin{remark} When $s \ll d$, the rate implies the need for samples that are exponential in the dimension $d$. However, we would like to clarify that this situation arises only for highly non-smooth functions. It is worth noting that in previous works on self-consuming generative models \cite{bertrand2023stability,alemohammad2023self}, the commonly assumed multivariate Gaussian distribution satisfies $s = d$, in which case our rate $O(n_i^{-\frac{s}{2s+2d}}) = O(n_i^{-\frac{1}{4}})$ does not suffer from the curse of dimensionality.
    
\end{remark}

%\subsection{General Data Cycle: Flexibly Regulating Data Composition Mixture}\label{general bound}
%This framework allows us to flexibly control the compositional mixture between real and synthetic data in the training distribution $p_i$. Our theoretical analysis quantifies the impacts of this adaptive training approach with mixed data on the cumulative error and the fidelity of models in self-consuming loops.
Next, we propose the TV distance bound for KDE within self-consuming loops.

\begin{theorem}\label{general theorem}
%Suppose $TV(p_{\theta_{i+1}},\widehat{p}_i)=\epsilon_{n_i}$ holds for all $i$. 
Let $s\geq 1$. we suppose $p_i \in W^{s,1}(\mathbb{R}^d)$, and for some $\epsilon>0$, $\int |x|^{d+\epsilon}p_i(x)\ dx<\infty$ for all $i$. Let $n_{i}$ represent the number of training samples obtained from the distribution $p_{i}$ at the $i$-th generation. We define $p_i$ as $p_{i} = \beta^1_{i}\widehat{p}_{1}+\beta^2_{i}\widehat{p}_{2}+ \cdots+\beta^{i}_{i}\widehat{p}_{i}+\alpha_{i}p_0$. Then, with probability at least $1-\delta$,
    \begin{align}
    &TV(\widehat{p}_{i+1},p_0)\notag \\
    &\lesssim \sum_{k=0}^i A_{i-k} \left(n_{i-k}^{-\frac{s}{2s+2d}}\sqrt{\log(i/\delta)}+n_{i-k}^{-\frac{2s+d}{4s+4d}}\right),\notag
    \end{align}
   where $A_i=1, A_{i-k}=\sum_{j=i-k+1}^i \beta_j^{i-k+1} A_j$ for $1\leq k \leq i$ and $\lesssim$ hides universal positive constants that depend solely on $K$, $p_i$, and $s$.
\end{theorem}

\textbf{Proof sketch of Theorem \ref{general theorem}}. To establish an upper bound on the TV distance between the non-parametric distribution $\widehat{p}_{i+1}$ of  future generative models and the original real data distribution $p_0$, we introduce an effective decomposition for $TV(\widehat{p}_{i+1}, p_0)$:
$$
TV(\widehat{p}_{i+1}, p_0) \leq TV(\widehat{p}_{i+1},p_i)+TV(p_{i},p_0).
$$
Firstly, the distance between $\widehat{p}_{i+1}$ and $p_i$ is bounded using kernel density estimation theory and concentration inequalities per Lemma \ref{lemma}, showing that it decays as $n_i^{-s/(2s+2d)}\sqrt{\log(2/\delta)} + n_i^{-(2s+d)/(4s+4d)}$ with high probability. Secondly, the distance between the mixed distribution $p_i$ and the original distribution $p_0$ is recursively expanded and bounded by a weighted sum of distances between the synthetic data distribution generated by previous generative models and the original distribution. By combining the bounds for the three components, the objective $TV(\widehat{p}_{i+1}, p_0)$ can be upper bounded by a term decaying as $n_i^{-s/(2s+2d)}\sqrt{\log(2/\delta)} + n_i^{-(2s+d)/(4s+4d)}$ plus a weighted sum of errors from previous generations. The final result can be obtained by solving it recursively.

The key steps involve utilizing kernel density estimators to control the error in estimating $p_i$ from finite samples, leveraging the triangle inequality to decompose the overall distance, and recursively bounding the distance between mixed training distributions across generations.

\section{Conclusion}
 %In this paper, we present a comprehensive theoretical analysis on the emerging challenge of training generative models within a self-consuming loop. Through the application of kernel density estimation theory and efficient decomposition techniques, we derive TV distance bounds between future learned synthetic data distributions and the original real data distribution under various mixed training scenarios. Notably, we show the necessity of quartic sample growth or incorporating high proportions of real data in the latest generations to constrain errors in the extreme full synthetic case. We further analyze the dynamics of increasing synthetic data on error propagation, unveiling a phase transition as synthetic data expands while real data remains fixed. Finally, we apply our framework to diffusion models and demonstrate optimal early-stopping strategies to control error accumulation. 

 In this paper, we addressed the emerging challenge of training generative models within a self-consuming loop, where successive generations of models are recursively trained on mixtures of real and synthetic data from previous generations. We constructed a robust theoretical framework to evaluate the impact of this training procedure on the data distributions learned by future models, encompassing both parametric and non-parametric models.

Our work specifically focused on deriving bounds on the TV distance between the synthetic data distributions produced by future models and the original real data distribution under various mixed training scenarios. For diffusion models with a one-hidden-layer neural network score function, we demonstrated that the TV distance could be effectively controlled by ensuring that mixed training dataset sizes or proportions of real data are sufficiently large. Furthermore, our analysis revealed a phase transition induced by the increasing amounts of synthetic data. We provided theoretical proof that while the TV distance initially increases, it eventually declines beyond a threshold point. This indicates that a balanced approach to incorporating synthetic data can enhance the performance of future generative models.

We also presented a detailed theoretical analysis of self-consuming non-parametric generative models. By employing kernel density estimation and efficient decomposition techniques, we derived TV distance bounds between future learned synthetic data distributions and the original real data distribution. In future work, it would be interesting to extend our theoretical results to biased sampling scenarios.

%In this paper, we have presented a theoretical analysis of training generative models in self-consuming loops using mixed real and synthetic data. Through rigorous bounds, we have shown that the discrepancy between synthetic and real data distributions can be controlled by expanding training sets or incorporating more real data over generations. We have also revealed phase transition dynamics as synthetic proportions increase. Moreover, for diffusion models, we have established tight bounds via optimal early stopping. Overall, our analysis offers important insights into the impacts of recycled synthetic data and provides guidance on training protocols to ensure stable generative modeling. Future work involves extending our framework to settings with biased sampling and other generative architectures.
\section*{Acknowledgements}
We appreciate Yuzhu Chen from the University of Science and Technology of China for his assistance in experiments. Dr Tao's research is partially supported by NTU RSR and Start Up Grants. And this work is supported in part by National Natural Science Foundation of China under Grant Nos. 62222117 and  62306338, Independent Innovation Research Project of China University of Petroleum (East China) (No.23CX06033A), and the Open Research Fund of Engineering Research Center of Intelligent Technology for Agriculture, Ministry of Education (Grant number ERCITA-KF002).

\section*{Impact Statement}
This paper presents theoretical analyses that aim to advance the understanding of training generative models within self-consuming loops, where models are recursively trained on mixtures of real and synthetic data. A potential impact of this paper is guiding practitioners on the standardized use of synthetic data for training generative models. Using synthetic data can also help protect privacy in sensitive domains such as medical imaging. There are no unethical aspects of this paper or anticipated negative effects on society. This theoretical analysis aims to advance the field of machine learning through rigorous study of an emerging challenge.

\bibliography{example_paper}
\bibliographystyle{icml2024}

%%%%%%%%%%%%%%%%%%%%%%%%%%%%%%%%%%%%%%%%%%%%%%%%%%%%%%%%%%%%%%%%%%%%%%%%%%%%%%%
%%%%%%%%%%%%%%%%%%%%%%%%%%%%%%%%%%%%%%%%%%%%%%%%%%%%%%%%%%%%%%%%%%%%%%%%%%%%%%%
% APPENDIX
%%%%%%%%%%%%%%%%%%%%%%%%%%%%%%%%%%%%%%%%%%%%%%%%%%%%%%%%%%%%%%%%%%%%%%%%%%%%%%%
%%%%%%%%%%%%%%%%%%%%%%%%%%%%%%%%%%%%%%%%%%%%%%%%%%%%%%%%%%%%%%%%%%%%%%%%%%%%%%%
\newpage
\appendix
\onecolumn

\section{Self-Consuming Loop of Generative Models}
In this section, we present the algorithms describing the self-consuming training loop. 

\begin{algorithm}[H]
   \caption{Self-Consuming Loop of Generative Models}
   \label{alg:example}
\begin{algorithmic}
   \STATE {\bfseries Input:} generative model $\mathcal{G}$, proportional coefficients $\{\alpha_i\}_{i=1}^{I}$, $\{\{\beta^k_{i}\}_{k=1}^i\}_{i=1}^I$.
  \STATE \textbf{Initialize} $\mathcal{D}_0=\{x_i^j\}_{j=1}^{n_0}$, with $x_i^j \sim p_0$, $p_{\theta_1}=\mathcal{G}_0(\mathcal{D}_0)$.
   \FOR{$i=1$ {\bfseries to} $I$}
    \STATE \(p_{i}=\beta_{i}^1 p_{\theta_1} + \cdots + \beta_{i}^{i} p_{\theta_{i}} + \alpha_{i}p_0\)
    \STATE $\mathcal{D}_i=\{x_i^j\}_{j=1}^{n_i}$, with $x_i^j \sim p_i$
    \STATE $p_{\theta_{i}}=\mathcal{G}_i(\mathcal{D}_i)$
   \ENDFOR
 \STATE \textbf{Return} %$p_{\theta_{T+1}}$
\end{algorithmic}
\end{algorithm}

\section{Auxiliary Lemmas}
In this section, we introduce McDiarmid's inequality, a tool that enables us to establish a bound on the probability that a function of multiple independent random variables deviates from its expected value.
\begin{lemma}[McDiarmid's Inequality]\label{Mcdiarmid}
 Consider independent random variables $Z_1, \cdots, Z_n \in \mathcal{Z}$ and a mapping $\phi: \mathcal{Z}^n \rightarrow \mathbb{R}$. If, for all $i \in\{1, \cdots, n\}$, and for all $z_1, \cdots, z_n, z_i^{\prime} \in \mathcal{Z}$, the function $\phi$ satisfies
$$
\left|\phi\left(z_1, \cdots, z_{i-1}, z_i, z_{i+1}, \cdots, z_n\right)-\phi\left(z_1, \cdots, z_{i-1}, z_i^{\prime}, z_{i+1}, \cdots, z_n\right)\right| \leq c,
$$
then,
$$
P\left(|\phi\left(Z_1, \cdots, Z_n\right)-\mathbb{E} \phi\left(Z_1, \ldots, Z_n\right) \geq t|\right) \leq 2\exp \left(\frac{-2 t^2}{n c^2}\right).
$$
Furthermore, for any $\delta \in(0,1)$ the following inequality holds with probability at least $1-\delta$
$$
\left|\phi\left(Z_1, \ldots, Z_n\right)-\mathbb{E}\left[\phi\left(Z_1, \ldots, Z_n\right)\right]\right| \leq \frac{c  \sqrt{n \log (2 / \delta)}}{\sqrt{2}}.
$$
\end{lemma}

\section{Proof of Theorem \ref{diffusion model}}
In this section, we present the proof of Theorem \ref{diffusion model}, which establishes
TV distance bounds for diffusion models.
\begin{lemma}[Theorem 1 in \citet{song2021maximum}]\label{songtheo} Let $p_i$ be the data distribution, $\pi$ be a known prior distribution. Then, we have
\begin{align}
   KL \left( p_{i} \| p_{\hat{\theta}_{i+1} (\tau)} \right) \leq \tilde{{\mathcal{L}}} \left( \hat{\theta}_{i+1} (\tau); g^2(\cdot) \right) + KL \left( p_{i,T} \| \pi \right). \notag
\end{align}
\end{lemma}

\begin{lemma}[Lemma 4 in \citet{li2023generalization}]\label{lemm4}
For any \(\tau > 0\) and \(\theta, \tilde{\theta}\), we have
\[
\bar{\tilde{\mathcal{L}}}(\bar{\theta}(\tau)) - \bar{\tilde{\mathcal{L}}}(\bar{\theta})\lesssim \frac{\|\bar{s}_{0,\bar{\theta}_0}\|_{\mathcal{H}}^2 + \|\bar{s}_{0,\bar{\theta}}\|_{\mathcal{H}}^2}{\tau}, \quad \widetilde{\mathcal{L}}(\theta(\tau)) - \widetilde{\mathcal{L}}(\theta) \lesssim\frac{\|s_{0,\theta_0}\|_{\mathcal{H}}^2 + \|s_{0,\theta}\|_{\mathcal{H}}^2}{\tau}.
\]
\end{lemma}

\begin{lemma}[Theorem A.5 in \cite{wu2023implicit}]\label{rademacher}
Consider a function class $\mathcal{F}$ with $\sup_{x \in \mathcal{X}, f \in \mathcal{F}} |f(x)| \leq B$. For any $\delta \in (0, 1)$, w.p. at least $1 - \delta$ over the choice of $S = (x_1, x_2, \ldots, x_n)$, we have,
\[
\left| \frac{1}{n} \sum_{i=1}^{n} f(x_i) - \mathbb{E}_x[f(x)] \right| \lesssim \mathcal{R}_n(\mathcal{F}) + B\sqrt{\frac{\ln(2/\delta)}{n}},
\]
where $\mathcal{R}_n(\mathcal{F})$ is the Rademacher complexity of $\mathcal{F}$.
\end{lemma}

\begin{lemma}[Lemma 5 in \citet{li2023generalization}]\label{lemma5} Suppose that the loss objectives $\widetilde{\mathcal{L}}$, $\widetilde{\mathcal{L}}_{n_i}$, $\bar{\tilde{\mathcal{L}}}$, $\bar{\tilde{\mathcal{L}}}_{n_i}$ are bounded at the initialization, then for any $\tau > 0$, we have
$$
\|s_{0,\theta(\tau)}\|_{\mathcal{H}},\ \|s_{0,\hat{\theta}_{i}(\tau)}\|_{\mathcal{H}} \lesssim \|s_{0,\theta_0}\|_{\mathcal{H}} + \sqrt{\tau/m_i},\quad \|\bar{s}_{0,\bar{\theta}(\tau)}\|_{\mathcal{H}}, \ \|\bar{s}_{0,\hat{\bar{\theta}}_i(\tau)}\|_{\mathcal{H}} \lesssim\|\bar{s}_{0,\bar{\theta}_0}\|_{\mathcal{H}} + \sqrt{\tau}.
$$
\end{lemma}

\begin{lemma}[Lemma 6 in \citet{li2023generalization}]\label{lemma6}
Suppose that \(\|x(0)\|_{\infty} \leq 1\), and the trainable parameter \(a\) and embedding function \(e(\cdot)\) are both bounded. Then, given any \(\bar{\theta}\), for any \(\delta > 0\), with the probability of at least \(1 - \delta\), there exists \(\theta\) such that
\[
\mathbb{E}_{t \sim \mathcal{U}(0,T)} \left[ \lambda(t) \cdot \mathbb{E}_{x_i(t) \sim p_{i,t}} \left[ \|s_{t,\theta}(x_i(t)) - \bar{s}_{t,\bar{\theta}}(x_i(t))\|_2^2 \right] \right]\lesssim \frac{\log^2(1/\delta^2)}{m_i} d ,
\]
where $\lesssim$ hides universal positive constants only depending on \(T\).
\end{lemma}

\begin{proof}[Proof of Theorem \ref{diffusion model}]
By utilizing the triangle inequality of TV distance, we can deduce the following:
\begin{equation}
 TV(p_{\theta_{i+1}},p_0)\leq TV(p_{\theta_{i+1}},p_i)+TV(p_{i},p_0).\notag
\end{equation}
For the second term, $TV(p_{i},p_0)$, recalling the definition of the TV distance as  $TV(p_i, p_0)=\frac{1}{2}\int |p_i(x)-p_0(x)|\ dx$ and given $p_i(x)=\beta_i^1 p_{\theta_1}(x)+\cdots+\beta_i^i p_{\theta_i}(x)+\alpha_i p_0(x)$, we can deduce:
\begin{align}
 TV(p_{\theta_{i+1}},p_0)\leq TV(p_{\theta_{i+1}}, p_i)+\beta^i_{i}TV(p_{\theta_i},p_0)+\beta^{i-1}_{i}TV(p_{\theta_{i-1}},p_0)+\cdots+\beta^{1}_{i}TV(p_{\theta_{1}},p_0). \label{32}
\end{align}
Next, let's concentrate on the first term, $TV(p_{\theta_{i+1}}, p_i)$. By applying Pinsker's inequality, we obtain: 
\begin{align}
    TV(p_{\theta_{i+1}}, p_i)\leq \sqrt{\frac{1}{2}KL(p_i \| p_{\theta_{i+1}})}.\label{pinsker}
\end{align}
Considering the training process of the diffusion model, in the $i+1$-th generation, the output is denoted as $\theta_{i+1}$, which we also represent as $\hat{\theta}_{i+1}(\tau_{i+1})$, with a total of $n_{i+1}$ training samples and a training time of $\tau_{i+1}$. Then, by employing Lemma \ref{songtheo}, we derive:
\begin{align}
    KL \left( p_i \| p_{\hat{\theta}_{i+1}}(\tau_{i+1}) \right) \leq \widetilde{\mathcal{L}}(\hat{\theta}_{i+1}(\tau_{i+1}); g^2(\cdot)) + KL \left( p_{i,T}\| \pi \right). \label{song}
\end{align}
In our approach, we adhere to the framework established by \citet{li2023generalization} to establish bounds for the term $\widetilde{\mathcal{L}}(\hat{\theta}_{i+1}(\tau_{i+1}); g^2(\cdot))$. For this purpose, we employ the following decomposition:
\begin{align}
\widetilde{\mathcal{L}}(\hat{\theta}_{i+1}(\tau_{i+1})) &= \left[ \widetilde{\mathcal{L}}(\hat{\theta}_{i+1}(\tau_{i+1})) - \widetilde{\mathcal{L}}(\theta_{i+1}(\tau_{i+1})) \right] + \widetilde{\mathcal{L}}(\theta_{i+1}(\tau_{i+1})) \notag \\
&\leq \left[ \widetilde{\mathcal{L}}(\hat{\theta}_{i+1}(\tau_{i+1})) - \widetilde{\mathcal{L}}(\theta_{i+1}(\tau_{i+1})) \right] + \left[\widetilde{\mathcal{L}}(\theta_{i+1}(\tau_{i+1}))-\bar{\tilde{\mathcal{L}}}(\bar{\theta}_{i+1}(\tau_{i+1}))\right]+\bar{\tilde{\mathcal{L}}}(\bar{\theta}_{i+1}(\tau_{i+1})).\label{decomp diffus}
\end{align}
Initially, regarding the final term $\bar{\tilde{\mathcal{L}}}(\bar{\theta}_{i+1}(\tau_{i+1}))$, we can derive the following result by employing Lemma \ref{lemm4}:
\begin{align}
   \bar{\tilde{\mathcal{L}}}(\bar{\theta}_{i+1}(\tau_{i+1}))\lesssim\bar{\tilde{\mathcal{L}}}(\bar{\theta}_{i+1}^*) + \frac{1}{\tau} \left( \|\bar{\mathbf{s}}_{0,\bar{\theta}_0}\|_{\mathcal{H}}^2 + \|\bar{\mathbf{s}}_{0,\bar{\theta}^*}\|_{\mathcal{H}}^2 \right). \label{decom1}
\end{align}
For the second term, $\widetilde{\mathcal{L}}(\theta_{i+1}(\tau_{i+1}))-\bar{\tilde{\mathcal{L}}}(\bar{\theta}_{i+1}(\tau_{i+1}))$, we perform the following decomposition:
\begin{align}
\widetilde{\mathcal{L}}(\theta_{i+1}(\tau_{i+1}))&-\bar{\tilde{\mathcal{L}}}(\bar{\theta}_{i+1}(\tau_{i+1}))\notag\\
\leq&\mathbb{E}_{t\sim\mathcal{U}(0,T)} \left[ \lambda(t) \cdot \mathbb{E}_{x_i(t)\sim p_{i, t}} \left[ \|\mathbf{s}_{t,\theta_{i+1}(\tau_{i+1})} (x_i(t)) - \bar{\mathbf{s}}_{t,\bar{\theta}_{i+1}(\tau_{i+1})} (x_i(t))\|^2 \right] \right] \notag\\
\lesssim& \mathbb{E}_{t\sim\mathcal{U}(0,T)} \left[ \lambda(t) \cdot \mathbb{E}_{x_i(t)\sim p_{i, t}} \left[ \|\mathbf{s}_{t,\theta_{i+1}(\tau_{i+1})} (x_i(t)) - \mathbf{s}_{t,\theta_{i+1}^*} (x_i(t))\|^2 \right] \right] \notag\\
&+\mathbb{E}_{t\sim\mathcal{U}(0,T)} \left[ \lambda(t) \cdot \mathbb{E}_{x_i(t)\sim p_{i, t}} \left[ \|\mathbf{s}_{t,\theta_{i+1}^*} (x_i(t)) - \bar{\mathbf{s}}_{t,\bar{\theta}_{i+1}^*} (x_i(t))\|^2 \right] \right] \notag\\
&+\mathbb{E}_{t\sim\mathcal{U}(0,T)} \left[ \lambda(t) \cdot \mathbb{E}_{x_i(t)\sim p_{i, t}} \left[ \|\bar{\mathbf{s}}_{t,\bar{\theta}_{i+1}^*} (x_i(t)) - \bar{\mathbf{s}}_{t,\bar{\theta}_{i+1}(\tau_{i+1})} (x_i(t))\|^2 \right] \right]\notag \\
=: & I_1+I_2+I_3,
\end{align}
where $\theta_{i+1}^*$ is the Monte Carlo estimator of $\bar{\theta}_{i+1}^*$. According to Lemma \ref{lemma6}, it is derived that, with a probability of at least $1-\delta$,
\begin{align}
    I_2\lesssim \frac{\log^2(1/\delta^2)}{m_i} d.
\end{align}
From the triangle inequality and Lemma \ref{lemm4}, we can obtain:
\begin{align}
I_{3} \lesssim \bar{\tilde{\mathcal{L}}}(\bar{\theta}_{i+1}(\tau_{i+1})) + \bar{\tilde{\mathcal{L}}}(\bar{\theta}_{i+1}^*)\lesssim\bar{\tilde{\mathcal{L}}}(\bar{\theta}_{i+1}^*) + \frac{1}{\tau} \left( \|\bar{\mathbf{s}}_{0,\bar{\theta}_0}\|_{\mathcal{H}}^2 + \|\bar{\mathbf{s}}_{0,\bar{\theta}^*}\|_{\mathcal{H}}^2 \right),\notag
\end{align}
and similarly,
\begin{align}
I_{1} \lesssim\widetilde{\mathcal{L}}(\theta^*_{i+1}) + \frac{1}{\tau} \left( \|\mathbf{s}_{0,\theta_0}\|_{\mathcal{H}}^2 + \|\mathbf{s}_{0,\theta^*}\|_{\mathcal{H}}^2 \right).\notag
\end{align}
Consequently, it can be established that, with a probability of at least $1 - \delta$,
\begin{align}
    \widetilde{\mathcal{L}}&(\theta_{i+1}(\tau_{i+1}))-\bar{\tilde{\mathcal{L}}}(\bar{\theta}_{i+1}(\tau_{i+1}))\notag \\
    &\lesssim\frac{\log^2(1/\delta^2)}{m_i}d+\widetilde{\mathcal{L}}(\theta^*_{i+1})+\bar{\tilde{\mathcal{L}}}(\bar{\theta}_{i+1}^*)+\frac{1}{\tau} \left( \|\bar{\mathbf{s}}_{0,\bar{\theta}_0}\|_{\mathcal{H}}^2 + \|\bar{\mathbf{s}}_{0,\bar{\theta}^*}\|_{\mathcal{H}}^2+\|\mathbf{s}_{0,\theta_0}\|_{\mathcal{H}}^2 + \|\mathbf{s}_{0,\theta^*}\|_{\mathcal{H}}^2\right).\label{decopm2}
\end{align}
For the first term $\widetilde{\mathcal{L}}(\hat{\theta}_{i+1}(\tau_{i+1})) - \widetilde{\mathcal{L}}(\theta_{i+1}(\tau_{i+1}))$, we have:
\begin{align}
    &\sqrt{\widetilde{\mathcal{L}}(\hat{\theta}_{i+1}(\tau_{i+1}))}-\sqrt{\widetilde{\mathcal{L}}(\theta_{i+1}(\tau_{i+1}))}\lesssim \left\{ \mathbb{E}_{t\sim\mathcal{U}(0,T)} \mathbb{E}_{x_i(t)\sim p_{i,t}} \left[ \lambda(t) \left\| s_{t,\hat{\theta}_{i+1}(\tau_{i+1})}(x_i(t)) - s_{t,\theta_{i+1}(\tau_{i+1})}(x_i(t)) \right\|^2 \right] \right\}^{\frac{1}{2}}=\notag\\
    &\left\{
    \mathbb{E}_{t\sim\mathcal{U}(0,T)} \mathbb{E}_{x_i(t)\sim p_{i,t}} \left[
        \lambda(t) \left\|
            \frac{1}{m_i} \sum_{j=1}^{m_i} \hat{\alpha}_j(\tau_{i+1}) \sigma(w_j^\top x_i(t) + u_j^\top e(t)) - 
            \frac{1}{m_i} \sum_{j=1}^{m_i} \alpha_j(\tau_{i+1}) \sigma(w_j^\top x_i(t) + u_j^\top e(t))
        \right\|^2
    \right]
\right\}^{\frac{1}{2}}\notag
\end{align}
Furthermore, by the triangle inequality, the Cauchy-Schwartz inequality, and the fact that $\sigma(y) = ReLU(y) \leq |y|$ for any $y\in\mathbb{R}$, Hölder's inequality, the positive homogeneity property of the ReLU activation, and the boundedness of the input data, embedding function $e(t)$, and weighting function $\lambda(t)$, we have:
\begin{align}
&\left\|\frac{1}{m_i} \sum_{j=1}^{m_i}(\hat{\alpha}_j(\tau_{i+1})- \alpha_j(\tau_{i+1}))\sigma(w_j^\top x_i(t) + u_j^\top e(t))\right\|^2\notag \\
& \leq \frac{1}{{m_i}^2} 
    \sum_{j=1}^{m_i} \left\| \hat{\alpha}_j(\tau_{i+1})- \alpha_j(\tau_{i+1}) \right\|^2\sum_{j=1}^{m_i}|\sigma\left(w_j^\top x_i(t) + u_j^\top e(t)\right)|^2 \notag \\
& \leq \frac{2}{{m_i}^2} 
    \sum_{j=1}^{m_i} \left\| \hat{\alpha}_j(\tau_{i+1})- \alpha_j(\tau_{i+1}) \right\|^2\sum_{j=1}^{m_i}(|w_j^\top x_i(t)|^2+|u_j^\top e(t)|^2) \notag \\
& \leq \frac{2}{{m_i}^2} 
    \sum_{j=1}^{m_i} \left\| \hat{\alpha}_j(\tau_{i+1})- \alpha_j(\tau_{i+1}) \right\|^2\sum_{j=1}^{m_i} ( \|w_j\|_1^2\|x_i(t)\|_{\infty}^2+\|u_j\|_1^2\|e(t)\|_{\infty}^2)\notag \\
&\lesssim\frac{1}{{m_i}^2} 
    \sum_{j=1}^{m_i} \left\| \hat{\alpha}_j(\tau_{i+1})- \alpha_j(\tau_{i+1}) \right\|^2\sum_{j=1}^{m_i}(C^2_{T,\delta}+C^2_{T,e}),
\end{align}
where $C_{T,\delta}$ and $C_{T,e}$ are constants arising from the boundedness of  $\|x\|_{\infty}$ and $e(t)$. Then, we have
\begin{align}
\sqrt{\widetilde{\mathcal{L}}(\hat{\theta}_{i+1}(\tau_{i+1}))}-\sqrt{\widetilde{\mathcal{L}}(\theta_{i+1}(\tau_{i+1}))}&\lesssim\left[\frac{1}{m_i} 
    \sum_{j=1}^{m_i} \left\| \hat{\alpha}_j(\tau_{i+1})- \alpha_j(\tau_{i+1}) \right\|^2(C^2_{T,\delta}+C^2_{T,e})\right]^{\frac{1}{2}}.\notag \\
&\lesssim(C_{T,\delta}+C_{T,e})\left[\frac{1}{m_i} 
    \sum_{j=1}^{m_i} \left\| \hat{\alpha}_j(\tau_{i+1})- \alpha_j(\tau_{i+1}) \right\|^2\right]^{\frac{1}{2}}.\notag
\end{align}
Furthermore, we conclude that:
\begin{align}
\widetilde{\mathcal{L}}&(\hat{\theta}_{i+1}(\tau_{i+1})) - \widetilde{\mathcal{L}}(\theta_{i+1}(\tau_{i+1}))\notag \\
\lesssim&\frac{1}{m_i} 
    \sum_{j=1}^{m_i} \left\| \hat{\alpha}_j(\tau_{i+1})- \alpha_j(\tau_{i+1}) \right\|^2(C^2_{T,\delta}+C^2_{T,e})+\sqrt{\widetilde{\mathcal{L}}(\theta_{i+1}(\tau_{i+1}))}(C_{T,\delta}+C_{T,e})\left[\frac{1}{m_i} 
    \sum_{j=1}^{m_i} \left\| \hat{\alpha}_j(\tau_{i+1})- \alpha_j(\tau_{i+1}) \right\|^2\right]^{\frac{1}{2}}\notag \\
\lesssim&\sqrt{\widetilde{\mathcal{L}}(\theta_{i+1}^*)+\frac{1}{\tau} \left( \|\mathbf{s}_{0,\theta_0}\|_{\mathcal{H}}^2 + \|\mathbf{s}_{0,\theta^*}\|_{\mathcal{H}}^2 \right)}(C_{T,\delta}+C_{T,e})\left[\frac{1}{m_i} 
    \sum_{j=1}^{m_i} \left\| \hat{\alpha}_j(\tau_{i+1})- \alpha_j(\tau_{i+1}) \right\|^2\right]^{\frac{1}{2}}\notag \\ 
    &+\frac{1}{m_i} 
    \sum_{j=1}^{m_i} \left\| \hat{\alpha}_j(\tau_{i+1})- \alpha_j(\tau_{i+1}) \right\|^2(C^2_{T,\delta}+C^2_{T,e}), \label{39}
\end{align}
where the last equality follows from lemma \ref{lemm4}. We further deduce that:
\begin{align}
&\frac{1}{m_i}\sum_{j=1}^{m_i}\left\|\hat{\alpha}_j(\tau_{i+1})- \alpha_j(\tau_{i+1})\right\|_2^2 \notag\\
&=\frac{1}{m_i}\sum_{j=1}^{m_i}\left\|\int_0^{\tau_{i+1}}\frac d{d\tau}(\hat{\alpha}_j(\tau)-\alpha_j(\tau))d\tau\right\|_2^2 \notag\\
&=\frac{1}{m_i}\sum_{j=1}^{m_i}\left\|\int_0^{\tau_{i+1}}\left(\nabla_{\theta^j_{i+1}(\tau)}\tilde{\mathcal{L}}(\theta_{i+1}(\tau))-\nabla_{\hat{\theta}^j_{i+1}(\tau)}\hat{\tilde{\mathcal{L}}}_{n_i}(\hat{\theta}_{i+1}(\tau))\right)d\tau\right\|_2^2 \notag\\
&=\frac1{{m_i}^2}\sum_{j=1}^{m_i}\left\|\int_0^{\tau_{i+1}}\left(2\mathbb{E}_{t\sim\mathcal{U}(0,T)}\left[\lambda(t)\mathbb{E}_{x_i(t)\sim p_{i,t}}\left[\left(s_{t,\theta_{i+1}(\tau)}(x_i(t))-\nabla_{x_i(t)}\log p_{i,t}(x_i(t))\right)\sigma(w_j^\top x_i(t)+u_j^\top e(t))\right]\right]\right.\right. \notag \\
&\left.\left.-2\mathbb{E}_{t\sim\mathcal{U}(0,T)}\left[\lambda(t)\mathbb{E}_{x_i(t)\sim \hat{p}_{i,t}}\left[\left(s_{t,\hat{\theta}_{i+1}(\tau)}(x_i(t))-\nabla_{x_i(t)}\log p_{i,t}(x_i(t))\right)\sigma(w_j^{\top}x_i(t)+u_j^{\top}e(t))\right]\right]\right)d\tau\right\Vert_2^2, \label{40}
\end{align}
where $\hat{p}_{i,t}$ denotes the standard empirical distribution of $p_{i,t}$. Note that:
\begin{align}
    \| \theta \|_2^2 = \| \text{vec}(A) \|_2^2 / m = \| A \|_F^2 / m = \| s_{0,\theta} \|_{\mathcal{H}}^2,
\end{align}
By lemma \ref{lemma5}, we get $\|\theta(\tau)\|_2=\|s_{0,\theta(\tau)}\|_{\mathcal{H}}\lesssim\|s_{0,\theta_0}\|_{\mathcal{H}} + \sqrt{\tau/m}$. For any $t \in [0, T]$, we define the function space as follows:
\begin{align}
    \mathcal{F}_t:=\{f_1(x_i(t);\theta_1(\tau))f_2(x_i(t);\theta_2):f_1\in\mathcal{F}_{1,t},f_2\in\mathcal{F}_{2,t}\},
\end{align}
where 
\begin{align}
&\mathcal{F}_{1,t} :=\left\{s_{t,\theta_{i+1}(\tau)}(x_i(t))-\nabla_{x_i(t)}\log p_{i,t}(x_i(t)):\|\theta_{i+1}(\tau)\|_2\lesssim\|s_{0,\theta_0}\|_{\mathcal{H}}+\sqrt{\tau/m_i}\right\},  \notag\\
&\mathcal{F}_{2,t} :=\left\{\sigma(w^\top x_i(t)+u^\top e(t)):\|w\|_1+\|u\|_1\leq1\right\}. \notag
\end{align}
Subsequently, according to Lemma \ref{rademacher}, for any $\delta \in (0,1)$, there exists a probability of at least $1 - \delta$ that the dataset $\mathcal{D}_{x_i} = \{x^j_i\}^{n_i}_{j=1}$ chosen satisfies the following: 
\begin{align}
&\mathbb{E}_{x_i(t)\sim p_{i,t}}\left[\left(s_{t,\theta_{i+1}(\tau)}(x_i(t))-\nabla_{x_i(t)}\log p_{i,t}(x_i(t))\right)\sigma(w_j^\top x_i(t)+u_j^\top e(t))\right] \notag \\
&- \mathbb{E}_{x_i(t)\sim \hat{p}_{i,t}}\left[\left(s_{t,\hat{\theta}_{i+1}(\tau)}(x_i(t))-\nabla_{x_i(t)}\log p_{i,t}(x_i(t))\right)\sigma(w_j^{\top}x_i(t)+u_j^{\top}e(t))\right] \notag\\
&\leq \left|\mathbb{E}_{x_i(t)\sim p_{i,t}}\left[\left(s_{t,\theta_{i+1}(\tau)}(x_i(t))-\nabla_{x_i(t)}\log p_{i,t}(x_i(t))\right)\sigma(w_j^\top x_i(t)+u_j^\top e(t))\right] \right. \notag\\
&\quad - \left. \mathbb{E}_{x_i(t)\sim \hat{p}_{i,t}}\left[\left(s_{t,\theta_{i+1}(\tau)}(x_i(t))-\nabla_{x_i(t)}\log p_{i,t}(x_i(t))\right)\sigma(w_j^\top x_i(t)+u_j^\top e(t))\right]\right| \notag\\
&\quad + \left| \mathbb{E}_{x_i(t)\sim \hat{p}_{i,t}}\left[\left(s_{t,\theta_{i+1}(\tau)}(x_i(t))-\nabla_{x_i(t)}\log p_{i,t}(x_i(t))\right)\sigma(w_j^\top x_i(t)+u_j^\top e(t))\right] \right. \notag\\
&\quad - \left. \mathbb{E}_{x_i(t)\sim \hat{p}_{i,t}}\left[\left(s_{t,\hat{\theta}_{i+1}(\tau)}(x_i(t))-\nabla_{x_i(t)}\log p_{i,t}(x_i(t))\right)\sigma(w_j^{\top}x_i(t)+u_j^{\top}e(t))\right] \right| \notag\\
&\lesssim \mathcal{R}_{n_i}(\mathcal{F}_t) + \sup_{f\in\mathcal{F}_t, x_i(t)\in[-C_{T,\delta},C_{T,\delta}]^d} \left| f(x_i(t)) \right| \sqrt{\frac{\log(2/\delta)}{n_i}} \notag\\
&\quad + \mathbb{E}_{x_i(t)\sim \hat{p}_{i,t}} \left[ \left( s_{t,\theta_{i+1}(\tau)}(x_i(t)) -  s_{t,\hat{\theta}_{i+1}(\tau)}(x_i(t)) \right) \sigma(w_j^\top x_i(t) + u_j^\top e(t)) \right] \quad := J_1 + J_2 + J_3, \label{43}
\end{align}
where $\mathcal{R}_{n_i}(\mathcal{F}_t)$ represents the empirical Rademacher complexity of the function space $\mathcal{F}_t$ on the dataset $\mathcal{D}_{x_i} = \{x^j_i\}^{n_i}_{j=1}$. For the $J_1$ component, as per Lemma A.6 in \citet{wu2023implicit}, we obtain: 
\begin{align}
\mathcal{R}_{n_i}(\mathcal{F}_t) \leq \left( \sup_{f_1 \in \mathcal{F}_{1,t}, x_i(t) \in [-C_{T,\delta}, C_{T,\delta}]^d} |f_1(x_i(t))| + \sup_{f_2 \in \mathcal{F}_{2,t}, x_i(t) \in [-C_{T,\delta}, C_{T,\delta}]^d} |f_2(x_i(t))| \right)\left( \mathcal{R}_{n_i}(\mathcal{F}_{1,t}) + \mathcal{R}_{n_i}(\mathcal{F}_{2,t}) \right). \label{rad1}
\end{align}
Note that
\begin{align}
|\sigma(w^\top x_i(t) + u^\top e(t))|\lesssim\|w\|_1\|x_i(t)\|_{\infty}+\|u\|_1\|e(t)\|_{\infty}\lesssim C_{T,\delta} + C_{T,e}. \label{45}
\end{align}
Then we obtain:
\begin{align}
\left| s_{t,\theta_{i+1}(\tau)}(x_i(t)) \right| &= \left| \frac{1}{\sqrt{m_i}} \sum_{j=1}^{m_i} \theta_{i+1}^j(\tau) \sigma(w_j^\top x_i(t) + u_j^\top e(t)) \right|\notag\\
&\lesssim (C_{T,\delta} + C_{T,e}) \left\| \theta_{i+1}(\tau) \right\|_2 \notag\\
&\lesssim (C_{T,\delta} + C_{T,e}) \left( \left\| s_{0,\theta_0} \right\|_{\mathcal{H}} + \sqrt{\tau/m_i} \right), \label{46}
\end{align}
Subsequently, we define $C'_{T,\delta} := \max_{x_i(t) \in [-C_{T,\delta}, C_{T,\delta}]^d} \left| \nabla_{x_i(t)} \log p_{i,t}(x_i(t)) \right|.$ Then, we can deduce that:
\begin{align}
    &|f_1(x_i(t))|=|s_{t,\theta_i+1}(\tau) (x_i(t)) - \nabla_{x_i(t)} \log p_{i,t}(x_i(t))|\lesssim (C_{T,\delta} + C_{T,e}) \left( \left\| s_{0,\theta_0} \right\|_{\mathcal{H}} + \sqrt{\tau/m_i} \right)+C'_{T,\delta},\notag\\
    &|f_2(x_i(t))|=\sigma(w^\top x_i(t) + u^\top e(t))\lesssim C_{T,\delta} + C_{T,e}.\notag
\end{align}
By substituting the above two equations into inequality \ref{rad1}, we can obtain
\begin{align}
    \mathcal{R}_{n_i}(\mathcal{F}_t)\lesssim(C_{T,\delta} + C_{T,e}) \left( \left\| s_{0,\theta_0} \right\|_{\mathcal{H}} + \sqrt{\tau/m_i}+1 \right)\left( \mathcal{R}_{n_i}(\mathcal{F}_{1,t}) + \mathcal{R}_{n_i}(\mathcal{F}_{2,t}) \right). \label{49}
\end{align}
Define $\mathcal{F}'_{1,t} :=\left\{s_{t,\theta_{i+1}(\tau)}(x_i(t)):\|\theta_{i+1}(\tau)\|_2\lesssim\|s_{0,\theta_0}\|_{\mathcal{H}}+\sqrt{\tau/m_i}\right\}$. According to Lemma 26.6 in \cite{shalev2014understanding}, we get $\mathcal{R}_{n_i}(\mathcal{F}_{1,t})\leq\mathcal{R}_{n_i}(\mathcal{F}'_{1,t})$. Then, by the definition of Rademacher complexity, we obtain:
\begin{align}
    \mathcal{R}_{n_i}(\mathcal{F}'_{1,t}) &= \frac{1}{n_i}\mathbb{E}_{\boldsymbol{\xi}} \left[ \sup_{\|\theta_{i+1}(\tau)\|_2\lesssim\|s_{0,\theta_0}\|_{\mathcal{H}}+\sqrt{\tau/m_i}} \left| \sum_{j=1}^{n_i} \xi_j s_{t,\theta_{i+1}(\tau)}(x_i^j(t)) \right| \right]\notag \\
    &\lesssim\frac{1}{n_i}\left( \left\| s_{0,\theta_0} \right\|_{\mathcal{H}} + \sqrt{\tau/m_i} \right)\mathbb{E}_{\boldsymbol{\xi}}\left[\sup_{\|w\|_1 + \|u\|_1 \leq 1} \left|\sum_{j=1}^{n_i} \xi_j  \sigma(w^\top x_i^j(t) + u^\top e(t))
 \right| \right].\notag
\end{align}
where the $\{\xi_j\}_{j=1}^{n_i}$ are independent random variables with the distribution $\mathbb{P}(\xi_j = 1) = \mathbb{P}(\xi_j = -1) = \frac{1}{2}$. Furthermore, leveraging the principle of symmetry, we derive:
\begin{align}
\mathbb{E}_{\boldsymbol{\xi}}\left[\frac{1}{n_i} \sup_{\|w\|_1 + \|u\|_1 \leq 1} \left|\sum_{j=1}^{n_i} \xi_j  \sigma(w^\top x_i^j(t) + u^\top e(t))
 \right| \right]&\leq 2\mathbb{E}_{\boldsymbol{\xi}}\left[ \frac{1}{n_i}\sup_{\|w\|_1 + \|u\|_1 \leq 1} \sum_{j=1}^{n_i} \xi_j \sigma(w^\top x_i^j(t) + u^\top e(t))\right]\notag \\
 &=2\mathcal{R}_{n_i}(\mathcal{F}_{2,t}).\notag
\end{align}
Then, we can get $\mathcal{R}_{n_i}(\mathcal{F}_{1,t})\leq\mathcal{R}_{n_i}(\mathcal{F}'_{1,t})\lesssim (\left\| s_{0,\theta_0} \right\|_{\mathcal{H}} + \sqrt{\tau/m_i})\mathcal{R}_{n_i}(\mathcal{F}_{2,t})$. According to Lemma 26.9 (Contraction lemma) and Lemma 26.11 in \citet{shalev2014understanding}, we have
\begin{align}
    \mathcal{R}_{n_i}(\mathcal{F}_{2,t})) \leq \left(\|x_i(t)\|_{\infty} + \|e(t)\|_{\infty}\right) \sqrt{\frac{2 \log(4d)}{n_i}}\lesssim (C_{T,\delta} + C_{T,e}) \sqrt{\frac{\log d}{n_i}}. \notag
\end{align}
Combining the aforementioned results with Equation \ref{49}, we subsequently obtain:
\begin{align}
    J_1=\mathcal{R}_{n_i}(\mathcal{F}_{t}))\lesssim(C_{T,\delta} + C_{T,e})^2(\left\| s_{0,\theta_0} \right\|_{\mathcal{H}} + \sqrt{\tau/m_i}+1)^2\sqrt{(\log d)/n_i}. \label{J1}
\end{align}
For $J_2$, by equation \ref{45} and \ref{46}, we have:
\begin{align}
    |f(x_i(t))|&=|(s_{t,\theta_i+1}(\tau)(x_i(t)) - \nabla_{x_i(t)} \log p_{i,t}(x_i(t)))| |\sigma(w_j^\top x_i(t) + u_j^\top e(t))|\notag \\
    &\lesssim(C_{T,\delta} + C_{T,e})^2 (\|s_{0,\theta_0}\|_{\mathcal{H}} + \sqrt{\tau/m_i}) + C'_{T,\delta}(C_{T,\delta} + C_{T,e}).\notag
\end{align}
Thus, we derive:
\begin{align}
    J_2=\sup_{f\in\mathcal{F}_t, x_i(t)\in[-C_{T,\delta},C_{T,\delta}]^d} \left| f(x_i(t)) \right| \sqrt{\frac{\log(2/\delta)}{n_i}}\lesssim(C_{T,\delta} + C_{T,e})^2(\|s_{0,\theta_0}\|_{\mathcal{H}} + \sqrt{\tau/m_i}+1)\sqrt{\frac{\log(1/\delta)}{n_i}}.\label{J2}
\end{align}
Similarly, in the case of $J_3$, we have:
\begin{align}
    J_3&= \mathbb{E}_{x_i(t)\sim \hat{p}_{i,t}} \left[ \left( s_{t,\theta_{i+1}(\tau)}(x_i(t)) -  s_{t,\hat{\theta}_{i+1}(\tau)}(x_i(t)) \right) \sigma(w_j^\top x_i(t) + u_j^\top e(t)) \right]\notag \\
    &\lesssim(C_{T,\delta} + C_{T,e})^2(\|s_{0,\theta_0}\|_{\mathcal{H}} + \sqrt{\tau/m_i}).\label{J3}
\end{align}
By integrating Equation \ref{40} with Equation \ref{43}, we derive the following result:
\begin{align}
\frac{1}{m_i}\sum_{j=1}^{m_i}\left\|\hat{\alpha}_j(\tau_{i+1})-\alpha_j(\tau_{i+1}) \right\|_2^2&\lesssim\frac{1}{{m_i}^2} \sum_{i=1}^{m_i} \left\| \int_{0}^{\tau_{i+1}} \mathbb{E}_{t\sim u(0,T)} [\lambda(t) (J_1 + J_2 + J_3)\boldsymbol{1}_d] d\tau \right\|_2^2\notag\\
&\lesssim(J_1 + J_2 + J_3)^2\tau_{i+1}^2\frac{d}{m_i}.\notag
\end{align}
Upon substituting Equations \ref{J1}, \ref{J2}, and \ref{J3} into the aforementioned equation, we subsequently obtain:
\begin{align}
    \frac{1}{m_i}\sum_{j=1}^{m_i}\left\|\hat{\alpha}_j(\tau_{i+1})-\alpha_j(\tau_{i+1}) \right\|_2^2\lesssim\tau_{i+1}^2\frac{d}{m_i}(C_{T,\delta} + C_{T,e})^4\left((\|s_{0,\theta_0}\|_{\mathcal{H}}^4 +\frac{\tau_{i+1}^2}{{m_i}^2}+1)\frac{\log(d/\delta)}{n_i}+\|s_{0,\theta_0}\|_{\mathcal{H}}^2 +\frac{\tau_{i+1}}{m_i}\right). \notag
\end{align}
By inserting the previously discussed equation into Equation \ref{39}, we then obtain:
\begin{align}
&\widetilde{\mathcal{L}}(\hat{\theta}_{i+1}(\tau_{i+1})) - \widetilde{\mathcal{L}}(\theta_{i+1}(\tau_{i+1}))\notag \\
&\lesssim\tau_{i+1}\sqrt{\frac{d}{m_i}}\left(\sqrt{\widetilde{\mathcal{L}}(\theta_{i+1}^*)}+\frac{1}{\sqrt{\tau}}(\|s_{0,\theta_0}\|_{\mathcal{H}} + \|s_{0,\theta_*}\|_{\mathcal{H}})\right)\left((\|s_{0,\theta_0}\|_{\mathcal{H}}^2 +\frac{\tau_{i+1}}{m_i}+1)\sqrt{\frac{\log(d/\delta)}{n_i}}+\|s_{0,\theta_0}\|_{\mathcal{H}} +\sqrt{\frac{\tau_{i+1}}{m_i}}\right) \notag \\
& \quad+\tau_{i+1}^2\frac{d}{m_i}\left((\|s_{0,\theta_0}\|_{\mathcal{H}}^4 +\frac{\tau_{i+1}^2}{{m_i}^2}+1)\frac{\log(d/\delta)}{n_i}+\|s_{0,\theta_0}\|_{\mathcal{H}}^2 +\frac{\tau_{i+1}}{m_i}\right)
 \notag \\
&\lesssim\tau_{i+1}\sqrt{\frac{d}{m_i}}\left((\frac{\|A_0\|_{F}^2}{m_i} +\frac{\tau_{i+1}}{m_i}+1)\sqrt{\frac{\log(d/\delta)}{n_i}}+\frac{\|A_0\|_{F}}{\sqrt{m_i}} +\sqrt{\frac{\tau_{i+1}}{m_i}}\right)\left(\sqrt{\widetilde{\mathcal{L}}(\theta_{i+1}^*)}+\frac{\|A_0\|_{F}+\|A_*\|_{F}}{\sqrt{m_i}}\right.  \notag \\
&\left. \quad+\tau_{i+1}\sqrt{\frac{d}{m_i}}\left((\frac{\|A_0\|_{F}^2}{m_i} +\frac{\tau_{i+1}}{m_i}+1)\sqrt{\frac{\log(d/\delta)}{n_i}}+\frac{\|A_0\|_{F}}{\sqrt{m_i}} +\sqrt{\frac{\tau_{i+1}}{m_i}}\right)\right)\notag \\
&\lesssim \tau_{i+1}\sqrt{\frac{d}{m_i}}\sqrt{\widetilde{\mathcal{L}}(\theta_{i+1}^*)}+\frac{\tau_{i+1}^2d\log(d/\delta)}{m_in_i}+\frac{\tau_{i+1}^4d\log(d/\delta)}{{m_i}^3n_i}+\frac{\tau_{i+1}^3d}{{m_i}^2}.
\end{align}
By incorporating the above equation along with Equations \ref{decom1} and \ref{decopm2} into Equation \ref{decomp diffus}, we consequently obtain:
\begin{align}
\widetilde{\mathcal{L}}(\hat{\theta}_{i+1}(\tau_{i+1}))\lesssim\frac{\tau_{i+1}^2d\log(d/\delta)}{mn_i}+\frac{\tau_{i+1}^4d\log(d/\delta)}{m^3n_i}+\frac{\tau_{i+1}^3d}{{m_i}^2}+\frac{\log^2(1/\delta^2)}{m_i}d+\widetilde{\mathcal{L}}(\theta^*_{i+1})+\bar{\tilde{\mathcal{L}}}(\bar{\theta}_{i+1}^*)+\frac{1}{\tau_{i+1}}.
\end{align}
Through the integration of the aforementioned equation with Equations \ref{song} and \ref{pinsker}, we consequently arrive at the following result:
\begin{align}
&TV(p_{\theta_{i+1}}, p_i)\notag \\
&\lesssim\sqrt{\frac{\tau_{i+1}^2d\log(d/\delta)}{m_in_i}}+\sqrt{\frac{\tau_{i+1}^4d\log(d/\delta)}{{m_i}^3n_i}}+\sqrt{\frac{\tau_{i+1}^3d}{{m_i}^2}}+\sqrt{\frac{1}{\tau_{i+1}}}+\sqrt{\widetilde{\mathcal{L}}(\theta^*_{i+1})}+\sqrt{\bar{\tilde{\mathcal{L}}}(\bar{\theta}_{i+1}^*)}+\sqrt{KL(p_{i,T}\|\pi)}.\notag
\end{align}
Let $m_i\asymp n_i$.  Upon choosing $\tau_{i+1} \asymp n_i^{1/2}$, we omit terms $\sqrt{\widetilde{\mathcal{L}}(\theta^*_{i+1})}$ and $\sqrt{\bar{\tilde{\mathcal{L}}}(\bar{\theta}_{i+1}^*)}$. Consequently, we obtain the following result:
\begin{align}\label{diffusion generalization bound}
TV(p_{\theta_{i+1}}, p_i)\lesssim\frac{(d\log(d/\delta))^{1/2}}{n_i^{1/4}}+\sqrt{KL(p_{i,T}\|\pi)}.
\end{align}
Upon substituting the above equation into Equation \ref{32}, we obtain:
\begin{align}
 TV(p_{\theta_{i+1}},p_0)\lesssim(d\log(d/\delta)^\frac{1}{2})/n_i^{\frac{1}{4}}+\sqrt{KL(p_{i,T}\|\pi)}+\beta^i_{i}TV(p_{\theta_i},p_0)+\cdots+\beta^{1}_{i}TV(p_{\theta_{1}},p_0). 
\end{align}
Analogous to the proof analysis process of Theorem \ref{general theorem}, we can derive the following through recursive methods, with probability at least $1-\delta$:
\begin{align}
    TV(p_{\theta_{i+1}},p_0)\lesssim \sum_{k=0}^i A_{i-k} \left(n_{i-k}^{-\frac{1}{4}}\sqrt{d\log \frac{di}{\delta}}+\sqrt{KL(p_{i-k,T}\|\pi)}\right),\notag
    \end{align}
   where $A_i=1, A_{i-k}=\sum_{j=i-k+1}^i \beta_j^{i-k+1} A_j$ for $1\leq k \leq i$. The proof is completed.

\end{proof}

\section{Proof of Corollary \ref{full synthetic data cycle}}
In this section, we present the proof of Corollary \ref{full synthetic data cycle}, which considers the most extreme case of full synthetic data.
\begin{proof}[Proof of Corollary \ref{full synthetic data cycle}] Similar to the proof process of Theorem \ref{diffusion model}, we first decompose using the triangle inequality:
 \begin{align}
    TV(p_{\theta_{i+1}},p_0)\leq TV(p_{\theta_{i+1}}, p_i)+TV(p_i, p_0).
\end{align}
Utilizing inequality \ref{diffusion generalization bound}, we can establish the following with a probability of at least $1-\delta$: 
\begin{align}
   TV(p_{\theta_{i+1}}, p_i)
   %&\leq n_i^{-\frac{s}{2s+2d}}\sqrt{\frac{1}{2}(\int |K|)^2\log \frac{2}{\delta}}+\frac{1}{2}n_i^{-\frac{s}{2s+2d}}\varphi(s,K,p_i)+\frac{1}{2}n_i^{-\frac{2s+d}{4s+4d}} (1+\gamma_{n_i})\sqrt{ \int K^2}\int \sqrt{p_i}\notag \\
   \lesssim \frac{(d\log(d/\delta))^{1/2}}{n_i^{1/4}}+\sqrt{KL(p_{i,T}\|\pi)}, \notag
\end{align}
Given that the definition of $p_i$ in the full synthetic data cycle is $p_i=p_{\theta_i}$, we can deduce the following with a probability of at least $1-\delta$:
\begin{align}
TV(p_{\theta_{i+1}},p_0)&\lesssim \frac{(d\log(d/\delta))^{1/2}}{n_i^{1/4}}+\sqrt{KL(p_{i,T}\|\pi)}+TV(p_{i},p_0).\notag \\ 
& =\frac{(d\log(d/\delta))^{1/2}}{n_i^{1/4}}+\sqrt{KL(p_{i,T}\|\pi)}+TV(p_{\theta_i},p_0) \notag
\end{align}
Since the initial term $TV(p_{\theta_1},p_0)$ can be expressed as follows:
\begin{align}
    TV(p_{\theta_1},p_0)\lesssim  \frac{(d\log(d/\delta))^{1/2}}{n_0^{1/4}}+\sqrt{KL(p_{0,T}\|\pi)}.\notag
\end{align}
Consequently, by solving recursively, we can obtain the following result, with probability at least $1-\delta$:
\begin{align}
    TV(p_{\theta_{i+1}},p_0)\lesssim \sum_{k=1}^i\left(\frac{(d\log(di/\delta))^{1/2}}{n_k^{1/4}}+\sqrt{KL(p_{k,T}\|\pi)}\right). \notag
\end{align}
The proof is completed.
\end{proof}

\section{Proof of Corollary \ref{balanced data corollary }}
In this section, we present the proof of Corollary \ref{balanced data corollary }, which examines the scenario of balanced data cycle.

\begin{proof}[Proof of Corollary \ref{balanced data corollary }]
Utilizing the triangle inequality of the TV distance, we can deduce the following: 
\begin{align}
    TV(p_{\theta_{i+1}},p_0)\leq TV(p_{\theta_{i+1}}, p_i)+TV(p_i, p_0).
\end{align}
By employing inequality \ref{diffusion generalization bound}, we can establish the following with a probability of at least $1-\delta$:
\begin{align}
     TV(p_{\theta_{i+1}},p_0)\lesssim \frac{(d\log(d/\delta))^{1/2}}{n_i^{1/4}}+\sqrt{KL(p_{i,T}\|\pi)}+TV(p_i, p_0).\label{coroally 4 1}
\end{align}
Recalling the definition of TV distance, we initially expound on the term $TV(p_{i},p_0)$ as follows:
\begin{align}
TV(p_i,p_0)&=\frac{1}{2}\int |p_i(x)-p_0(x)|\ dx. \notag\\
&\leq \frac{1}{2}\int |\frac{1}{i+1}(p_0(x)+p_{\theta_1}(x)+p_{\theta_2}(x)+\cdots+p_{\theta_i}(x))-p_0(x)|\ dx \notag\\
&\leq \frac{1}{2}\frac{1}{i+1}\int |p_{\theta_1}(x)-p_0(x)|+|p_{\theta_2}(x)-p_0(x)|+\cdots+|p_{\theta_i}(x)-p_0(x)| \ dx\notag \\
&\leq \frac{1}{i+1} \sum_{j=1}^iTV(p_{\theta_j}, p_0).
\end{align}

Substituting the aforementioned inequality into Inequality \ref{coroally 4 1}, we obtain:
\begin{align}
     TV(p_{\theta_{i+1}},p_0)\lesssim \frac{(d\log(d/\delta))^{1/2}}{n_i^{1/4}}+\sqrt{KL(p_{i,T}\|\pi)}+\frac{1}{i+1} \sum_{j=1}^iTV(p_{\theta_j}, p_0).\notag
\end{align}
Specify the function $f(n_i)$ in the following manner: $f(n_i)=\frac{(d\log(d/\delta))^{1/2}}{n_i^{1/4}}+\sqrt{KL(p_{i,T}\|\pi)}$. As a result, we derive the subsequent outcome:
\begin{equation}\label{theorem 212 1}
 TV(p_{\theta_{i+1}},p_0)\lesssim f(n_i)+\frac{1}{i+1}TV(p_{\theta_i},p_0)+\frac{1}{i+1}TV(p_{\theta_{i-1}},p_0)+\cdots+\frac{1}{i+1}TV(p_{\theta_1},p_0).
\end{equation}
From the aforementioned expression, we can additionally infer:
$$
 TV(p_{\theta_i},p_0)\lesssim f(n_{i-1})+\frac{1}{i}TV(p_{\theta_{i-1}},p_0)+\frac{1}{i}TV(p_{\theta_{i-2}},p_0)+\cdots+\frac{1}{i}TV(p_{\theta_1},p_0).
$$
By multiplying both sides of the above equation by $\frac{1}{i+1}$, we obtain the following result:
$$
\frac{1}{i+1}TV(p_{\theta_i},p_0)\lesssim  \frac{1}{i+1}f(n_{i-1})+ \frac{1}{i+1}\frac{1}{i}TV(p_{\theta_{i-1}},p_0)+\cdots+ \frac{1}{i+1}\frac{1}{i}TV(p_{\theta_1},p_0).
$$
Plugging the above inequality into inequality \ref{theorem 212 1}, we obtain:
\begin{align}
&TV(p_{\theta_{i+1}},p_0)\notag\\ 
&\lesssim f(n_i)+\frac{1}{i+1}f(n_{i-1})+(\frac{1}{i+1}+\frac{1}{i+1}\frac{1}{i})TV(p_{\theta_{i-1}},p_0)+\cdots+(\frac{1}{i+1}+\frac{1}{i+1}\frac{1}{i})TV(p_{\theta_1},p_0).
\end{align}
Note that $TV(p_{\theta_1},p_0)\lesssim \frac{(d\log(d/\delta))^{1/2}}{n_0^{1/4}}+\sqrt{KL(p_{0,T}\|\pi)}$, By solving recursively,
\begin{align}
TV(p_{\theta_{i+1}},p_0) \lesssim &f(n_i)+\frac{1}{i+1}f(n_{i-1})+(\frac{1}{i+1}+\frac{1}{i+1}\frac{1}{i})f(n_{i-2})+\cdots+\notag \\
&(\frac{1}{i+1}+\frac{1}{i+1}\frac{1}{i}+\cdots+\frac{1}{i+1}\frac{1}{i}\cdots\frac{1}{2})f(n_{0}).
\end{align}
Thus, we have, with probability at least $1-\delta$,
\begin{align}
 &TV(p_{\theta_{i+1}},p_0) \notag \\ 
 &\lesssim \frac{(d\log(di/\delta))^{1/2}}{n_i^{1/4}}+\sqrt{KL(p_{i,T}\|\pi)}+\sum_{k=0}^{i-1}\sum_{j=k}^{i-1}  \frac{\Gamma(j+2)}{\Gamma(i+2)}\left(\frac{(d\log(di/\delta))^{1/2}}{n_k^{1/4}}+\sqrt{KL(p_{k,T}\|\pi)}\right), \notag
\end{align}
where the Gamma function $\Gamma(j)=(j-1)!$ and $j$ is a positive integer. The proof is complete.
\end{proof}

\section{Proof of Corollary \ref{phase transition}}
In this section, we present the proof of Corollary \ref{phase transition}, which analyzes the phase transition phenomena in error dynamics when increasing synthetic data while keeping real data fixed.
\begin{proof}[Proof of Corollary 5] 
By utilizing the triangle inequality of the TV distance, we can derive the following:
\begin{align}
    TV(p_{\theta_{i+1}},p_0)\leq TV(p_{\theta_{i+1}}, p_i)+TV(p_i, p_0).
\end{align}
By employing inequality \ref{diffusion generalization bound}, we can establish the subsequent statement with a probability of at least $1-\delta$: 
\begin{align}
     TV(p_{\theta_{i+1}},p_0)\lesssim \frac{(d\log(d/\delta))^{1/2}}{(n+m)^{1/4}}+\sqrt{KL(p_{i,T}\|\pi)}+TV(p_i, p_0).\label{coroally 41 1}
\end{align}
Recalling the definition of the total variation distance, let's first elaborate on the expression $TV(p_{i},p_0)$ as follows:
\begin{align}
TV(p_i,p_0)&=\frac{1}{2}\int |p_i(x)-p_0(x)|\ dx. \notag\\
&\leq \frac{1}{2}\int |\frac{n}{n+m}p_0(x)+\frac{m}{n+m}p_{\theta_i}(x)-p_0(x)|\ dx \notag\\
&\leq \frac{1}{2}\frac{m}{n+m}\int |p_{\theta_i}(x)-p_0(x)| \ dx\notag \\
&\leq \frac{m}{n+m} TV(p_{\theta_i}, p_0).
\end{align}

Substituting the aforementioned inequality into Inequality \ref{coroally 41 1}, we obtain:
\begin{align}
     TV(p_{\theta_{i+1}},p_0)\lesssim \frac{(d\log(d/\delta))^{1/2}}{(n+m)^{1/4}}+\sqrt{KL(p_{i,T}\|\pi)}+\frac{m}{n+m} TV(p_{\theta_i}, p_0).\notag
\end{align}
Note that $TV(p_{\theta_1},p_0)\lesssim \frac{(d\log(d/\delta))^{1/2}}{(n+m)^{1/4}}+\sqrt{KL(p_{i,T}\|\pi)}$. Similar to the proof of Corollary \ref{balanced data corollary }, we can solve recursively, with probability at least $1-\delta$,
\begin{align}
 TV(p_{\theta_{i+1}},p_0) \leq \left(1+\frac{m}{n}\right)\left(1-(\frac{m}{n+m})^{i+1}\right)\left(\frac{(d\log(di/\delta))^{1/2}}{(n+m)^{1/4}}+\sqrt{KL(p_{i,T}\|\pi)}\right). \notag
\end{align}
The proof is completed.
\end{proof}

\section{Proof of Lemma \ref{lemma}}
In this section, we present the proof of Lemma \ref{lemma}. The proof utilizes the kernel density estimation framework along with McDiarmid's inequality. We first define the class $s$ kernel as follows:
\begin{definition}\label{kernel definition}
Let $s\geq 1$. A class $s$ kernel is a Borel measurable function $K$ which satisfies
\begin{itemize}
    \item K is symmetric, i.e., $K(-x)=K(x),\ x\in\mathbb{R}^d$.
    \item $\int K=1$.
    \item $\int x^{\alpha}K(x)\ dx=0$ for $1\leq |\alpha|\leq s-1$.
    \item $\int |x^{\alpha}||K(x)|\ dx<\infty$ for $|\alpha|= s$.
    \item $\int (1 + \|x\|^{d+\epsilon}) K(x)^2 \, dx < \infty$ for some $\epsilon>0$
\end{itemize}
 \end{definition}

\begin{proof}[Proof of lemma \ref{lemma}]
Revisiting the definition of the kernel density estimation $\widehat{p}_{i}$, we derive the following:
$$
\widehat{p}_{i+1}(x) = \frac{1}{n_ih_i^d} \sum_{j=1}^{n_i} K \left( \frac{x - x_j}{h_i} \right),
$$
where $K$ belongs to the class $s$. Define $K_{h_i}(u)=(1/h_i^d)K(u/h_i)$. Our objective is to establish a bound for $TV(\widehat{p}_{i+1},p_i)=\frac{1}{2} \int |\widehat{p}_{i+1}(x)-p_i(x)|\ dx$. Consider two sets: $\{x_1, \cdots, x_{n_i}\}$ and $\{x^{\prime}_1, \cdots , x^{\prime}_{n_i}\}$, where $x^{\prime}_j = x_j$ for all $j$ except when $j = t$. As a result, we derive:
\begin{align}
\bigg| \frac{1}{2}\int | \widehat{p}_{i+1}(x; x_1, \ldots, x_n) - p_i(x)| \, dx  &-  \frac{1}{2}\int |\widehat{p}_{i+1}(x; x_1', \ldots, x_n') - p_i(x) | \, dx \bigg| \notag\\
&\leq  \frac{1}{2}\int \left| \widehat{p}_{i+1}(x; x_1, \ldots, x_n) - \widehat{p}_{i+1}(x; x_1', \ldots, x_n') \right| \, dx \notag\\
&\leq \frac{1}{2n_i} \int \left| K_{h_i}(x - x_t) - K_{h_i}(x - x_t') \right| \, dx \notag \\
&\leq \frac{1}{n_ih_i^d} \int |K|. 
\end{align}
Since $\frac{1}{2} \int |\widehat{p}_{i+1}(x)-p_i(x)|\ dx$ is $ \frac{\int |K|}{n_ih_i^d}$-Lipschitz under the Hamming metric, the application of McDiarmid's Inequality (lemma \ref{Mcdiarmid}) yields the ensuing result with a probability of at least $1-\delta$:
\begin{align}
    \left|\frac{1}{2} \int |\widehat{p}_{i+1}(x)-p_i(x)|\ dx-\mathbb{E}\ \frac{1}{2} \int |\widehat{p}_{i+1}(x)-p_i(x)|\ dx \right|\leq \sqrt{\frac{(\int |K|)^2}{2n_ih_i^{2d}}\log \frac{2}{\delta}}. \notag
\end{align}
Then, we obtain:
\begin{align}
 TV(\widehat{p}_{i},p_i)=\frac{1}{2} \int |\widehat{p}_{i+1}(x)-p_i(x)|\ dx \leq   \frac{1}{2} \mathbb{E} \int |\widehat{p}_{i+1}(x)-p_i(x)|\ dx+\sqrt{\frac{(\int |K|)^2}{2n_ih_i^{2d}}\log \frac{2}{\delta}}. \label{decom11}
\end{align}
Next, our objective is to establish a bound for $\mathbb{E} \int |\widehat{p}_{i+1}(x)-p_i(x)|\ dx$. For simplicity, we omit the $(\cdot)$ and $dx$ in our notation. Through decomposition, we can derive the following results:
\begin{align}
\mathbb{E} \int |\widehat{p}_{i+1}-p_i| \leq  \int |p_i * K_{h_i} - p_i| + \mathbb{E}   \int |\widehat{p}_{i+1} - p_i * K_{h_i}|,\label{decomposition}
\end{align}
where $p_i * K_{h_i}(x) = \int p_i(y)K_{h_i}(x - y) \, dy$. The terms on the right-hand side will be called the bias and variation terms of the error. Assume first that $p_i{\in C_0^{\infty}(\mathbb{R}^d)}$, where $C_0^{\infty}(\mathbb{R}^d)$ denotes the space of infinitely differentiable functions with compact support. By Taylor's theorem, 
\begin{align}
p_i(x+y)-p_i(x)=\sum_{j=1}^{s-1}\sum_{|\alpha|=j}\frac{1}{\alpha!}y^{\alpha}\partial^{\alpha}p_i(x)+ 
\sum_{|\alpha|=s}\frac{s!}{\alpha!}\int_{0}^{1}\frac{(1-t)^{s-1}}{(s-1)!}y^{\alpha}\partial^{\alpha}p_i(x+ty)dt. \label{taylor}
\end{align}
Since $K$ is a symmetric kernel, we have $\int K(z) \, dz = 1$ and $\int zK(z) \, dz = 0$. Thus, for $x \in \mathbb{R}^d$, we derive:
\begin{align}
    p_i * K(x)-p_i(x)&=\int (p_i(y)-p_i(x))K(x - y) \, dy \notag\\
    &=\int (p_i(x-y)-p_i(x))K(y) \, dy \notag\\
    &=\int (p_i(x+y)-p_i(x))K(y) \, dy. \label{taylor2}
\end{align}
Substituting Equation \ref{taylor} into Equation \ref{taylor2} and taking into account that $\int x^{\alpha}K(x)\ dx=0$ for $1\leq |\alpha|\leq s-1$, we derive:
\begin{align}
     p_i * K(x)-p_i(x)=\sum_{|\alpha|=s}\frac{s!}{\alpha!}\int \int_{0}^{1}\frac{(1-t)^{s-1}}{(s-1)!}y^{\alpha}\partial^{\alpha}p_i(x+ty)K(y)\ dt\ dy. \label{lemma12}
\end{align}
The integrals with respect to $y$ exist due to the application of Fubini's theorem and the introduction of new variables $\eta = -ty$ and $\tau = t^{-1}$. This holds true for $|\alpha| = s$,
\begin{align}
      \int \int_{0}^{1}&\frac{(1-t)^{s-1}}{(s-1)!}|y^{\alpha}|\ |\partial^{\alpha}p_i(x+ty)|\ |K(y)|\ dt\ dy\notag\\
      &= \int_{0}^{1} \int \frac{(1-t)^{s-1}}{(s-1)!}|y^{\alpha}|\ |\partial^{\alpha}p_i(x+ty)|\ |K(y)|\ dy\ dt \notag \\
      & = \int_{0}^{1} \int \frac{(1-t)^{s-1}}{(s-1)!}t^{-s-d}|\eta^{\alpha}|\ |\partial^{\alpha}p_i(x-\eta)|\ |K(\frac{\eta}{t})|\ d\eta\ dt \notag\\
      &=\int|\partial^{\alpha}p_i(x-\eta)| \int_{1}^{\infty}  \frac{(\tau-1)^{s-1}}{(s-1)!}\tau^{d-1}|\eta^{\alpha}|\ |K(\tau\eta)|\ d\tau\ d\eta. \notag
\end{align}
Since the integral $\int |x^\alpha K(x)| dx$ is finite for $|\alpha| = s$, the last integral is also finite. By repeating the same steps, we can conclude: 
\begin{align}
    \int \int_{0}^{1}\frac{(1-t)^{s-1}}{(s-1)!}y^{\alpha}\partial^{\alpha}p_i(x+ty)K(y)\ dt\ dy=\int \partial^{\alpha}p_i(x-\eta) \ (-1)^s\int_{1}^{\infty}  \frac{(\tau-1)^{s-1}}{(s-1)!}\tau^{d-1}\eta^{\alpha}\ K(\tau\eta)\ d\tau\ d\eta. \notag
\end{align}
By defining the associated kernel $L^{\alpha}(x)$ as $(-1)^{|\alpha|}\int_{1}^{\infty} \frac{(t-1)^{|\alpha|-1}}{(|\alpha|-1)!}t^{d-1}x^{\alpha}\ K(t x)\ dt$, we can then obtain:
\begin{align}
     \int \int_{0}^{1}\frac{(1-t)^{s-1}}{(s-1)!}y^{\alpha}\partial^{\alpha}p_i(x+ty)K(y)\ dt\ dy= \partial^{\alpha}p_i*L^{\alpha} (x). \label{lemma 15}
\end{align}
Plugging Equation \ref{lemma 15} into the Equation \ref{lemma12}, then, we obtain:
\begin{align}
     p_i * K(x)-p_i(x)=\sum_{|\alpha|=s}\frac{s!}{\alpha!}\partial^{\alpha}p_i*L^{\alpha} (x). \label{sob}
\end{align}
When \( p_i \in W^{s,1}(\mathbb{R}^d) \), there are functions \( p_{i,n} \in C_0^\infty(\mathbb{R}^d) \), \( n = 0, 1, \ldots \), such that \( \partial^\alpha p_{i,n} \to \partial^\alpha p_i \) in \( L^1(\mathbb{R}^d) \) for \( |\alpha| \leq s \) \cite{adams2003sobolev}. Since Equation \ref{sob} holds for each \( p_{i,n} \), it also holds for \( p_i \) a.e. on \( \mathbb{R}^d \).
Furthermore, for the kernel $K_{h_i}$, applying Young's inequality, we obtain:
\begin{align}
    \int |p_i * K_{h_i}-p_i| \leq h_i^s \sum_{|\alpha|=s}\frac{s!}{\alpha!} \int |\partial^{\alpha}p_i| \int |L^{\alpha}|=h_i^s \varphi(s,K, p_i). \label{bias}
\end{align}
Given that $p_i \in W^{s,1}(\mathbb{R}^d)$ and $\int |x^{\alpha}||K(x)|\ dx<\infty$ for $|\alpha|= s$, it follows that $\varphi(s,K, p_i)=\sum_{|\alpha|=s}\frac{s!}{\alpha!} \int |\partial^{\alpha}p_i| \int |L^{\alpha}|$ is finite. Thus, the bias tends to zero at least at the rate $h_i^s$. Next, we address the variation term $\mathbb{E}   \int |\widehat{p}_{i+1} - p_i * K_{h_i}|$. By the Schwarz inequality, we obtain: 
\begin{align}
    \mathbb{E}   \int |\widehat{p}_{i+1} - p_i * K_{h_i}|\leq(n_ih_i^d)^{-\frac{1}{2}} \int \sqrt{p_i*(K^2)_{h_i}}.\notag
\end{align}
Let $Q=K^2/ \int K^2$, and by utilizing the inequality $\sqrt{p_i * Q_{h_i}} \leq \sqrt{p_i} + \sqrt{|p_i - p_i * Q_{h_i}|}$, we obtain:
\begin{align}
    \mathbb{E}   \int |\widehat{p}_{i+1} - p_i * K_{h_i}|\leq(n_ih_i^d)^{-\frac{1}{2}} (1+\frac{\int \sqrt{|p_i-p_i*Q_{h_i}|}}{\int \sqrt{p_i}})\sqrt{ \int K^2}\int \sqrt{p_i}.\notag
\end{align}
Let $\gamma(h_i)=\int \sqrt{|p_i-p_i*Q_{h_i}|}/\int \sqrt{p_i}$. Then, our objective is to show $\gamma(h_i)\to 0$ as $h_i \to 0^{+}$. First, we apply Carlson's inequality to get 
\begin{align}
    \int \sqrt{|p_i - p_i * Q_{h_i}|} \leq C \left( \int |p_i - p_i * Q_{h_i}| \right)^{\epsilon/2(\epsilon+d)} \times \left( \int \|x\|^{d+\epsilon} |p_i(x) - p_i * Q_{h_i}(x)| \, dx \right)^{d/2(\epsilon+d)}, \notag
\end{align}
where $C$ is a constant.  Since $ \int |p_i - p_i * Q_{h_i}| $ tends to zero as \( h_i \rightarrow 0^+ \) [\cite{stein1970singular}, Chapt. III], our task is to demonstrate that the second integral remains bounded. Thus, we have
\begin{align}
    \int \|x\|^{d+\epsilon} |p_i(x) - p_i * Q_{h_i}(x)| \, dx \leq \int  \|x\|^{d+\epsilon} p_i(x)\ dx+ \int \|x\|^{d+\epsilon}p_i * Q_{h_i}(x)\ dx. \notag
\end{align}
According to the assumption, we have $\int  \|x\|^{d+\epsilon} p_i(x)\ dx$ is finite. For the second integral, we utilize \(\xi = x - y\), \(\|\xi + y\|^{d+\epsilon} \leq 2^{d+\epsilon- 1}(\|\xi\|^{d+\epsilon} + \|y\|^{d+\epsilon})\), and the fact that \(\int p_i = \int Q_{h_i} = 1\), which implies that:
\begin{align}
\int \|x\|^{d+\epsilon} p_i * Q_{h_i}(x) \, dx 
&= \int \left( \int \|x\|^{d+\epsilon} p_i(x - y) \, dx \right) Q_{h_i}(y) \, dy  \notag\\
&\leq 2^{d+\epsilon-1} \int_{\mathbb{R}^d} \|\mathbf{\xi}\|^{d+\epsilon} p_i(\mathbf{\xi}) \, d\mathbf{\xi} + 2^{d+\epsilon-1} \int \|y\|^{d+\epsilon} Q_{h_i}(y) \, dy.\notag
\end{align}
Since $\int (1 + \|x\|^{d+\epsilon}) K(x)^2 \, dx < \infty \quad $, the above integral is finite. Thus, we have
\begin{align}
    \mathbb{E}   \int |\widehat{p}_{i+1} - p_i * K_{h_i}|\leq(n_ih_i^d)^{-\frac{1}{2}} (1+\gamma(h_i))\sqrt{ \int K^2}\int \sqrt{p_i},\label{variation}
\end{align}
where $\gamma(h_i)\to 0$ as $h_i \to 0^{+}$. Combining the inequality \ref{bias} for the bias term and inequality \ref{variation} for the variation term into the inequality \ref{decomposition}, we obtain:
\begin{align}
    \mathbb{E} \int |\widehat{p}_{i+1}-p_i| \leq h_i^s\varphi(s,K,p_i)+(n_ih_i^d)^{-\frac{1}{2}} (1+\gamma(h_i))\sqrt{ \int K^2}\int \sqrt{p_i}.
\end{align}
Plugging the above inequality into the inequality \ref{decom11}, then, we obtain:
\begin{align}
    TV(\widehat{p}_{i+1}, p_i)\leq (n_ih_i^{2d})^{-\frac{1}{2}}\sqrt{\frac{1}{2}(\int |K|)^2\log \frac{2}{\delta}}+\frac{1}{2}h_i^s\varphi(s,K,p_i)+\frac{1}{2}(n_ih_i^d)^{-\frac{1}{2}} (1+\gamma(h_i))\sqrt{ \int K^2}\int \sqrt{p_i}. \notag
\end{align}
By the choice of $h_i=n_i^{-\frac{1}{2s+2d}}$, then, we obtain:
\begin{align}
   TV(\widehat{p}_{i+1}, p_i)\leq n_i^{-\frac{s}{2s+2d}}\sqrt{\frac{1}{2}(\int |K|)^2\log \frac{2}{\delta}}+\frac{1}{2}n_i^{-\frac{s}{2s+2d}}\varphi(s,K,p_i)+\frac{1}{2}n_i^{-\frac{2s+d}{4s+4d}} (1+\gamma_{n_i})\sqrt{ \int K^2}\int \sqrt{p_i}.
\end{align}
Where $\gamma_{n_i} \to 0$ as $n_i \to \infty$ and $\varphi(s,K,p_i)$ is a finite function. The proof is completed.
\end{proof}

\section{Proof of Theorem \ref{general theorem}}
In this section, we present the proof of Theorem \ref{general theorem}, which targets the scenario of the general data cycle.
\begin{proof}[Proof of Theorem \ref{general theorem}]
By leveraging the triangle inequality of TV distance, we can deduce the following:
\begin{equation}\label{theorem 2 decomp}
 TV(\widehat{p}_{i+1},p_0)\leq TV(\widehat{p}_{i+1},p_i)+TV(p_{i},p_0).
\end{equation}
Revisiting the definition of TV distance, we initially elucidate the term $TV(p_{i},p_0)$ as follows:
\begin{align}
TV(p_i,p_0)&=\frac{1}{2}\int |p_i(x)-p_0(x)| dx \notag\\
&\leq \frac{1}{2}\int |\beta^1_{i}\widehat{p}_{1}(x)+\beta^2_{i}\widehat{p}_{2}(x)+ \cdots+\beta^{i}_{i}\widehat{p}_{i}(x)+\alpha_{i}p_0(x)-p_0(x)| \ dx. \notag
\end{align}
Given that $\beta_i^1+\beta_i^2+\cdots+\beta_i^i+\alpha_i=1$, we can deduce the following:
\begin{align}
  TV(p_i,p_0)&\leq   \frac{1}{2}\int |\beta^1_{i}\widehat{p}_{1}(x)+\beta^2_{i}\widehat{p}_{2}(x)+ \cdots+\beta^{i}_{i}\widehat{p}_{i}(x)-(\beta_i^1+\beta_i^2+\cdots+\beta_i^i)p_0(x)|\ dx\notag \\
  &\leq  \frac{1}{2}\int \beta_i^1 |\widehat{p}_{1}(x)-p_0(x)|+\beta_i^2 |\widehat{p}_{2}(x)-p_0(x)|+\cdots+\beta_i^i |\widehat{p}_{i}(x)-p_0(x)| \ dx\notag \\
  &=\beta^1_{i}TV(\widehat{p}_{1},p_0)+\beta^2_{i}TV(\widehat{p}_{2},p_0)+\cdots+\beta^{i}_{i}TV(\widehat{p}_{i},p_0)\label{lemma 1.2}.
\end{align}
By employing Lemma \ref{lemma}, and considering that $\gamma_{n_i} \to 0$ as $n_i \to \infty$, along with $\varphi(s,K,p_i)$ being a finite function, we can derive the following with a probability of at least $1-\delta$:
\begin{align}
   TV(\widehat{p}_{i+1}, p_i)&\leq n_i^{-\frac{s}{2s+2d}}\sqrt{\frac{1}{2}(\int |K|)^2\log \frac{2}{\delta}}+\frac{1}{2}n_i^{-\frac{s}{2s+2d}}\varphi(s,K,p_i)+\frac{1}{2}n_i^{-\frac{2s+d}{4s+4d}} (1+\gamma_{n_i})\sqrt{ \int K^2}\int \sqrt{p_i}\notag \\
   &\lesssim n_i^{-\frac{s}{2s+2d}}\sqrt{\log(2/\delta)}+n_i^{-\frac{2s+d}{4s+4d}}, \label{17}
\end{align}
where \(\lesssim\) conceals universal positive constants that depend solely on \(K\), \(p_i\) and $s$. By incorporating Inequality \ref{lemma 1.2} and \ref{17} into the Inequality \ref{theorem 2 decomp}, we can consequently derive the following with a probability of at least $1-\delta$:
$$
 TV(\widehat{p}_{i+1},p_0)\lesssim n_i^{-\frac{s}{2s+2d}}\sqrt{\log(2/\delta)}+n_i^{-\frac{2s+d}{4s+4d}}+\beta^i_{i}TV(\widehat{p}_{i},p_0)+\beta^{i-1}_{i}TV(\widehat{p}_{i-1},p_0)+\cdots+\beta^{1}_{i}TV(\widehat{p}_{1},p_0).
$$
Define the function  $f(n_i)$ as follows: $f(n_i)=n_i^{-\frac{s}{2s+2d}}\sqrt{\log(2/\delta)}+n_i^{-\frac{2s+d}{4s+4d}}$. Consequently, we obtain the ensuing result:
\begin{equation}\label{theorem 21 1}
 TV(\widehat{p}_{i+1},p_0)\lesssim f(n_i)+\beta^i_{i}TV(\widehat{p}_{i},p_0)+\beta^{i-1}_{i}TV(\widehat{p}_{i-1},p_0)+\cdots+\beta^{1}_{i}TV(\widehat{p}_{1},p_0).
\end{equation}
From the above expression, we can further deduce:
$$
 TV(\widehat{p}_{i},p_0)\lesssim f(n_{i-1})+\beta^{i-1}_{i-1}TV(\widehat{p}_{i-1},p_0)+\beta^{i-2}_{i-1}TV(\widehat{p}_{i-2},p_0)+\cdots+\beta^{1}_{i-1}TV(\widehat{p}_{1},p_0).
$$
By multiplying both sides of the above equation by $\beta_i^i$, we obtain the following result:
$$
 \beta_i^iTV(\widehat{p}_{i},p_0)\lesssim  \beta_i^if(n_{i-1})+ \beta_i^i\beta^{i-1}_{i-1}TV(\widehat{p}_{i-1},p_0)+ \beta_i^i\beta^{i-2}_{i-1}TV(\widehat{p}_{i-2},p_0)+\cdots+ \beta_i^i\beta^{1}_{i-1}TV(\widehat{p}_{1},p_0).
$$
Plugging the above inequality into inequality \ref{theorem 21 1}, we obtain:
$$
TV(\widehat{p}_{i+1},p_0)\lesssim f(n_i)+\beta^i_{i}f(n_{i-1})+(\beta^{i-1}_{i}+\beta_i^i\beta^{i-1}_{i-1})TV(\widehat{p}_{i-1},p_0)+\cdots+(\beta^{1}_{i}+\beta_i^i\beta^{1}_{i-1})TV(\widehat{p}_{1},p_0).
$$
Define the coefficient preceding the term $f(n_{i})$ as $A_i$. Consequently, we have $A_i=1$ and $A_{i-1}=\beta_i^i$. By adopting a procedure analogous to the one described above, we can derive the following result:
\begin{align}
    A_{i-2}&=\beta^{i-1}_{i}+\beta_i^i\beta^{i-1}_{i-1}=\beta^{i-1}_{i}A_i+\beta^{i-1}_{i-1}A_{i-1}, \notag \\
    A_{i-3}&=\beta^{i-2}_i+\beta^{i}_{i}\beta^{i-2}_{i-1}+(\beta^{i-1}_{i}+\beta^{i}_{i}\beta^{i-1}_{i-1})\beta^{i-2}_{i-2}=\beta^{i-2}_iA_i+\beta^{i-2}_{i-1}A_{i-1}+\beta^{i-2}_{i-2}A_{i-2}. \notag
\end{align}
Based on the aforementioned discussion, we can draw the following conclusion:
\begin{align}
    A_{i-t}=\beta_i^{i-t+1}A_i+\beta_{i-1}^{i-t+1}A_{i-1}+\cdots+\beta_{i-t+1}^{i-t+1}A_{i-t+1}=\sum_{j=i-t+1}^i \beta_j^{i-t+1} A_j, \quad 1\leq t \leq i. \label{theorem 12 2}
\end{align}
Furthermore, for the initial term $TV(\widehat{p}_{1},p_0)$, we have:
\begin{align}
    TV(\widehat{p}_{1},p_0)
    &\lesssim n_0^{-\frac{s}{2s+2d}}\sqrt{\log(2/\delta)}+n_0^{-\frac{2s+d}{4s+4d}} \notag \\
    &=f(n_0).    \label{theorem 21 3}
\end{align}
By synthesizing Inequalities \ref{theorem 21 1}, \ref{theorem 12 2}, and \ref{theorem 21 3}, we arrive at the following result, with probability at least $1-\delta$:
\begin{align}
    TV&(\widehat{p}_{i+1},p_0)\notag \\
    &\lesssim \sum_{k=0}^i A_{i-k} \left(n_{i-k}^{-\frac{s}{2s+2d}}\sqrt{\log(i/\delta)}+n_{i-k}^{-\frac{2s+d}{4s+4d}}\right), \notag \\
\end{align}
where $A_i=1, A_{i-k}=\sum_{j=i-k+1}^i \beta_j^{i-k+1} A_j$ for $1\leq k \leq i$.
\end{proof}

\section{Extensions}\label{extensions}
In this section, we present some extensions of our analyses.

\subsection{Extension to Normalizing Flows}

For Normalizing flows, we follow the setting in \citet{yang2022mathematical}. The theorem is extended as follows:

\begin{theorem}\label{normalizing flows}
Assuming the second moment of $p_i$ is finite for all $i$, we set the base distribution to be the unit Gaussian $\mathcal{N}$. Let the velocity field be modeled by $V$, which belongs to $\mathcal{H}(\mathbb{R}^{d+1},\mathbb{R}^d)$, a space of functions representable as $f_{\mathbf{a}}(\mathbf{x})=\mathbb{E}_{\rho(\mathbf{w},b)}\big[\mathbf{a}(\mathbf{w},b)\left.\sigma(\mathbf{w}\cdot\mathbf{x}+b)\right]$, where $\rho\in P(\mathbb{R}^{d+1})$ is a fixed parameter distribution and $\mathbf{a}\in L^2(\rho,\mathbb{R}^d)$ is a parameter function. Let $G_{V_i}$ denote the flow map defined as $G_{V_i}(x_i(0))=x_i(1),\quad\frac{d}{d\tau}x_i(\tau)=V_i(x_i(\tau),\tau)$, and let the reverse-time flow map for $\tau\in[0,1]$ be defined as $F_{V_i}(x_i(1),\tau)=x_i(\tau),\quad\frac{d}{d\tau}x_i(\tau)=V_i(x_i(\tau),\tau)$. Assuming all the flow-induced norm of optimal $V^*_i$ satisfies $\|V^*_i\|_{\mathcal{F}}=\exp\|V^*_i\|_{\mathcal{H}}\leq R$. Let $n_{i}$ be the number of training samples obtained from the distribution $p_{i}$. Then with probability at least \( 1 - \delta \), we establish that:
\begin{align}
    TV(p_{\theta_{i+1}},p_0)\lesssim \sum_{k=0}^i A_{i-k} \left(n_{i-k}^{-\frac{1}{4}}R\sqrt{1+R^2}\log^{\frac{1}{4}}\frac{i}{\delta}\right),\notag
\end{align}
where $A_i=1, A_{i-k}=\sum_{j=i-k+1}^i \beta_j^{i-k+1} A_j$ for $1\leq k \leq i$.
\end{theorem}

\begin{proof}[Proof of Theorem \ref{normalizing flows}]
Using the triangle inequality for the TV distance, we can infer the following:
\begin{equation}
 TV(p_{\theta_{i+1}},p_0)\leq TV(p_{\theta_{i+1}},p_i)+TV(p_{i},p_0).\notag
\end{equation}
For the second term, $TV(p_{i},p_0)$, considering the definition of TV distance as $\frac{1}{2}\int |p_i(x)-p_0(x)|\ dx$, and given $p_i(x)=\beta_i^1 p_{\theta_1}(x)+\cdots+\beta_i^i p_{\theta_i}(x)+\alpha_i p_0(x)$, we can infer:
\begin{align}
 TV(p_{\theta_{i+1}},p_0)\leq TV(p_{\theta_{i+1}}, p_i)+\beta^i_{i}TV(p_{\theta_i},p_0)+\beta^{i-1}_{i}TV(p_{\theta_{i-1}},p_0)+\cdots+\beta^{1}_{i}TV(p_{\theta_{1}},p_0). \label{extension32}
\end{align}
Next, let's focus on the first term, $TV(p_{\theta_{i+1}}, p_i)$. Utilizing Pinsker's inequality, we derive:
\begin{align}
    TV(p_{\theta_{i+1}}, p_i)\leq \sqrt{\frac{1}{2}KL(p_i \| p_{\theta_{i+1}})}.
\end{align}
Following the framework established by \citet{yang2022mathematical}, we establish bounds for the $KL$ term. Let $L$ be the population loss defined as:
\begin{align}
L(V_i)&=\iint_0^1\mathrm{Tr}\big[\nabla V_i(x_i(\tau),\tau)\big]d\tau+\frac12\|x_i(0)\|^2dp_i(x_i(1))\notag\\
\mathbf{x}_i(\tau)&:=G_\tau(G^{-1}(x_i(1))
\end{align}
In addition, let $L^{(n_i)}$ and $V^{(n_i)}_i$ be the corresponding empirical loss and the output. It follows that:
\begin{align}
L\left(V_i^{(n_i)}\right) & \leq L^{(n_i)}\left(V_i^{(n_i)}\right)+\sup _{\|V_i\|_{\mathcal{F}} \leq R} L(V_i)-L^{(n_i)}(V_i) \notag\\
& \leq L^{(n_i)}\left(V_i^*\right)+\sup _{\|V_i\|_{\mathcal{F}} \leq R} L(V_i)-L^{(n_i)}(V_i) \notag\\
& \leq L\left(V_i^*\right)+2 \sup _{\|V_i\|_{\mathcal{F}} \leq R} L(V_i)-L^{(n_i)}(V_i)\notag
\end{align}
Then, we obtain the following:
\begin{align}
& \sup _{\|V_i\|_{\mathcal{F}} \leq R} L(V_i)-L^{(n_i)}(V_i) \notag\\
\leq & \sup _{\|V_i\|_{\mathcal{F}} \leq R} \iint_0^1 \operatorname{Tr}\left[\nabla V_i\left(F_{V_i}\left(x_i(\tau), \tau\right), \tau\right)\right] d \tau d\left(p_i-\overline{p}_i\right)(x) \notag\\
& +\sup _{\|V_i\|_{\mathcal{F}} \leq R} \iint_0^1 \frac{1}{2}\left\|F_{V_i}(x, 1)\right\|^2 d\left(p_i-\overline{p}_i\right)(x)
\end{align}
Let $A$ and $B$ represent the two terms as random variables. Employing the techniques elucidated in [\citet{ma2019priori}, Theorem 2.11] and [\citet{han2021class}, Theorem 3.3] to bound the Rademacher complexity of flow-induced functions, we can obtain:
\begin{align}
&\mathbb{E}[A] \lesssim \frac{R}{\sqrt{n_i}} \mathbb{E}\left[\max _{1 \leq j \leq n_i}\left\|x_i^j\right\|\right] \lesssim \frac{R^2}{\sqrt{n_i}} \notag\\
& \mathbb{E}[B] \lesssim \frac{R^2}{\sqrt{n_i}} \mathbb{E}\left[\max _{1 \leq j \leq n_i}\left\|x_i^j\right\|^2\right] \lesssim \frac{R^4}{\sqrt{n_i}}\notag
\end{align}
Furthermore, concerning the variances $A-\mathbb{E}[A]$ and $B-\mathbb{E}[B]$, we can employ the extension of McDiarmid's inequality to sub-Gaussian random variables \cite{kontorovich2014concentration}, to demonstrate that, with a probability of $1-\delta$,
$$
A-\mathbb{E}[A] \lesssim \frac{R^2 \sqrt{\log 1 / \delta}}{\sqrt{n_i}} . \quad B-\mathbb{E}[B] \lesssim \frac{R^4 \sqrt{\log 1 / \delta}}{\sqrt{n_i}}
$$
Combining these inequalities, we induce that:
$$
KL(p_i \| p_{\theta_{i+1}})\lesssim  \frac{R^2(1+R^2)\sqrt{\log 1 / \delta}}{\sqrt{n_i}}
$$
Utilizing Pinsker's inequality, we derive:
\begin{align}
    TV(p_{\theta_{i+1}}, p_i)\lesssim n_{i}^{-\frac{1}{4}}R\sqrt{1+R^2}\log^{\frac{1}{4}}\frac{1}{\delta}.
\end{align}
Upon substituting the above equation into Equation \ref{extension32}, we obtain:
\begin{align}
 TV(p_{\theta_{i+1}},p_0)\lesssim n_{i}^{-\frac{1}{4}}R\sqrt{1+R^2}\log^{\frac{1}{4}}\frac{1}{\delta}+\beta^i_{i}TV(p_{\theta_i},p_0)+\cdots+\beta^{1}_{i}TV(p_{\theta_{1}},p_0). 
\end{align}
Analogous to the proof analysis process of Theorem \ref{general theorem}, we can derive the following through recursive methods, with probability at least $1-\delta$:
\begin{align}
    TV(p_{\theta_{i+1}},p_0)\lesssim \sum_{k=0}^i A_{i-k} \left(n_{i-k}^{-\frac{1}{4}}R\sqrt{1+R^2}\log^{\frac{1}{4}}\frac{i}{\delta}\right),\notag
\end{align}
where $A_i=1, A_{i-k}=\sum_{j=i-k+1}^i \beta_j^{i-k+1} A_j$ for $1\leq k \leq i$.

\end{proof}

\begin{remark}\textbf{Extension to Transformer Models.} Exploring the theoretical extension of this work to transformer models trained within a self-consumption loop presents significant challenges. Firstly, characterizing the generalization error of transformers at each generation is essential. It's worth noting that this task is particularly challenging due to the complexity of the training data, which consists of a mixture of data from different distributions. Furthermore, delving into how this error compounds with each generation poses an even more complex scenario. Investigating this intricate situation remains a subject for future work, requiring further exploration.
\end{remark}

\section{Experiments}
In this section, we present some experimental results. Specifically, we trained a diffusion model on the MNIST dataset. Consistent with previous works \cite{alemohammad2023self, bertrand2023stability}, we employed the FID score as a metric to evaluate model performance. We trained multiple generations under two scenarios: 1) a 1:1 ratio of real and synthetic data, and 2) fully synthetic data. This allowed us to investigate the impact of the number of training generations, the total number of training samples, and the presence of real data on model performance, as detailed in the table below:
\begin{center}
\begin{tabular}{|c|c|c|c|c|c|}  
\hline   \textbf{Dataset} & \textbf{gen01} & \textbf{gen02} & \textbf{gen03} & \textbf{gen04} \\   
\hline  \textbf{syn10k} & $120.99$ & $156.77$ & $165.04$ & $180.23$  \\ 
\hline  \textbf{syn15k} & $55.46$ & $137.45$ & $151.95$ & $157.62$  \\ 
\hline  \textbf{syn20k} & $24.44$ & $46.78$ & $61.14$ & $78.77$  \\ 
\hline  \textbf{mix10k} & $58.94$ & $117.73$ & $137.62$ & $136.39$  \\ 
\hline  \textbf{mix15k} & $55.69$ & $66.93$ & $69.63$ & $75.91$  \\ 
\hline  \textbf{mix20k} & $20.05$ & $31.92$ & $37.36$ & $36.13$  \\ 
\hline   
\end{tabular}   
\end{center}   
We observed that our experimental findings closely aligned with the theoretical results derived in the paper. When the number of training samples and the mix in the training set remained constant, the model performance deteriorated as the number of training generations increased. However, increasing the number of training samples and incorporating real data both led to improved model performance.

In addition, we fixed the amount of real data at 9K samples and varied the number of synthetic samples to 13K, 15K, 17K, and 20K. The experimental results are as follows: 
\begin{center}
\begin{tabular}{|c|c|c|c|c|c|}  
\hline   \textbf{Dataset} & \textbf{syn 13k} & \textbf{syn 15k} & \textbf{syn 17k} & \textbf{syn 20k} \\   
\hline  \textbf{real 9k} & $74.95$ & $87.15$ & $33.50$ & $28.39$  \\ 
\hline  
\end{tabular}   
\end{center}   

These findings corroborate the theoretical predictions in Remark \ref{unbias}, where we indicate that the TV distance first exhibits an ascent and then declines beyond a threshold point as synthetic data expands while real data remains fixed.

%%%%%%%%%%%%%%%%%%%%%%%%%%%%%%%%%%%%%%%%%%%%%%%%%%%%%%%%%%%%%%%%%%%%%%%%%%%%%%%
%%%%%%%%%%%%%%%%%%%%%%%%%%%%%%%%%%%%%%%%%%%%%%%%%%%%%%%%%%%%%%%%%%%%%%%%%%%%%%%

\end{document}